%% file: main.tex
\documentclass[journal]{IEEEtran}
\usepackage{xurl} 
\usepackage{algorithm}
\usepackage{algorithmic}
\usepackage[switch]{lineno}
\usepackage{verbatim}
\usepackage{subcaption}
\usepackage{multirow}
\usepackage{makecell}
\usepackage{pgfplots}
\usepgfplotslibrary{fillbetween}
\pgfplotsset{compat=newest}
\usepackage{pgfplotstable}
\usepackage{caption}
\usepackage{subcaption}
\usetikzlibrary{patterns}
\usepackage{soul}
\usepackage{amssymb}
\usepackage{booktabs}
\usepackage{amsthm}
\usepackage{amsmath}

\newtheorem{theorem}{Theorem}
\newtheorem{assumption}{Assumption}
\usepackage{pifont}% http://ctan.org/pkg/pifont
\newcommand{\cmark}{\ding{51}}%
\newcommand{\xmark}{\ding{55}}%

\usepackage{orcidlink}

\begin{document}
\title{FedUP: Efficient Pruning-based Federated Unlearning for Model Poisoning Attacks}

% REBECCA: Model Poisoning Attacks sufficiente abbastanza per considerare tutto? in ogni caso metterei forse "for mitigating?

\author{Nicolò~Romandini* \orcidlink{0000-0002-2820-5978},
%~\IEEEmembership{Member,~IEEE,}
        ~Cristian~Borcea \orcidlink{0000-0003-0020-0910},
        %~\IEEEmembership{Member,~IEEE,}
    ~Rebecca~Montanari~\orcidlink{0000-0002-3687-0361},
        ~Luca~Foschini \orcidlink{0000-0001-9062-3647}
        %~\IEEEmembership{Senior Member,~IEEE\vspace{-1cm}}% <-this % stops a space
%\thanks{}% <-this % stops a space
\thanks{Manuscript received ; ; . }%
\thanks{Nicolò Romandini, Rebecca Montanari, and Luca Foschini are with the Department of Computer Science and Engineering (DISI), University of Bologna, Bologna, Italy (e-mail: \{name.surname\}@unibo.it).\\
Cristian Borcea is with the New Jersey Institute of Technology, Newark, New Jersey, USA  (email: borcea@njit.edu).
}%
\thanks{Asterisk
indicates the corresponding author}}% <-this % stops a space
%\end{comment}


\markboth{Journal of \LaTeX\ Class Files,~Vol.~14, No.~8, August~2015}%
{Shell \MakeLowercase{\textit{et al.}}: Bare Demo of IEEEtran.cls for IEEE Journals}



\maketitle


% NEW ABSTACT FOR ARXIV
\begin{abstract}
    Federated Learning (FL) can be vulnerable to attacks, such as model poisoning, where adversaries send malicious local weights to compromise the global model. Federated Unlearning (FU) is emerging as a solution to address such vulnerabilities by selectively removing the influence of detected malicious contributors on the global model without complete retraining. However, unlike typical FU scenarios where clients are trusted and cooperative, applying FU with malicious and possibly colluding clients is challenging because their collaboration in unlearning their data cannot be assumed. This work presents FedUP, a lightweight FU algorithm designed to efficiently mitigate malicious clients' influence by pruning specific connections within the attacked model. Our approach achieves efficiency by relying only on clients' weights from the last training round before unlearning to identify which connections to inhibit. Isolating malicious influence is non-trivial due to overlapping updates from benign and malicious clients. FedUP addresses this by carefully selecting and zeroing the highest magnitude weights that diverge the most between the latest updates from benign and malicious clients while preserving benign information. FedUP is evaluated under a strong adversarial threat model, where up to $50\% - 1$ of the clients could be malicious and have full knowledge of the aggregation process. We demonstrate the effectiveness, robustness, and efficiency of our solution through experiments across IID and Non-IID data, under label-flipping and backdoor attacks, and by comparing it with state-of-the-art (SOTA) FU solutions. In all scenarios, FedUP reduces malicious influence, lowering accuracy on malicious data to match that of a model retrained from scratch while preserving performance on benign data. FedUP achieves effective unlearning while consistently being faster and saving storage compared to the SOTA.
\end{abstract}


% \fi
\IEEEpeerreviewmaketitle

%COMMENTI REBECCA GENERALI  PER LA INTRO: i contributi non sono focalizzati su anche prte di security, vanno a mio avviso riscritti e focalizzati; guardare commenti miei sparsi nella intro
\section{Introduction}
\label{introduction}

\IEEEPARstart{T}{he}  distributed nature of Federated Learning (FL) makes it susceptible to various attacks, including model poisoning attacks, where an adversary injects malicious local weights to disrupt the global model or poison it by introducing backdoors\cite{tolpegin2020data,xia2023poisoning}. %rebecca: vogliamo citare il survey che ti ho mandato su poisonong in attack in FL?
Malicious actors are becoming increasingly skilled at executing stealth attacks \cite{typhoon}. The risk escalates significantly when critical public infrastructures like power plants are targeted \cite{usa}. Many techniques have been developed to detect these attacks \cite{li2019abnormal}. %rebecca: c'è un survey o qualcosa che si puo' citare?
However, after detection, recovering from such attacks can be costly and time-consuming, often necessitating retraining the model from scratch. {\color{black} While retraining from scratch is a possible solution, it introduces substantial delays, as it requires restarting the entire training process. During this time, the compromised model cannot be safely used, and clients must wait for a new, clean version. This interruption can degrade service availability and user trust. } This highlights the need to design mechanisms to alleviate the burden of dealing with such situations.

Federated Unlearning (FU) \cite{wu2022federated} has emerged as a solution to these problems. FU refers to the process of selectively removing the impact of specific data contributions from an FL model without the necessity of complete retraining. Unlike traditional Machine Unlearning (MU) \cite{bourtoule2021machine}, which is applied to centralized models, FU addresses the unique challenges posed by the decentralized nature of FL. Effective FU techniques aim to remove specific data contributions efficiently while retaining the valuable knowledge gained from the remaining data. 
However, applying FU algorithms in a malicious context differs from the typical exploitation of FU in scenarios where clients are considered trusted and cooperative \cite{wu2022federatedkd,kassab2022federated,10269017}.
Some work \cite{10179336,guo2023fast} propose solutions in a malicious context, but they require the server to retain all client updates received during the training process, using them for performing unlearning. 
%These approaches are slow and storage-inefficient, leading to significant delays and increased resource consumption.
{\color{black} Storing gradients and model updates introduces for the entire duration of training both storage inefficiencies and substantial privacy risks, as a server breach could allow adversaries to exploit the stored updates to reconstruct private client data \cite{geiping2020inverting}.}
The challenge that we address is how to apply FU efficiently in a scenario with malicious clients, where it is not possible to rely on the collaboration of adversaries whose data needs to be unlearned. 
Many studies have already addressed detecting malicious clients \cite{li2020learning,zhang2022fldetector,jiang2023data}, but detection alone, while important, is not sufficient.
{\color{black}}
Once a malicious client is identified, the still open research question is how to remove its contribution from the FL model as soon as possible to minimize damage.
{\color{black} 
%While retraining from scratch is a possible solution, it introduces substantial delays, as it requires restarting the entire training process. During this time, the compromised model cannot be safely used, and clients must wait for a new, clean version. This interruption can degrade service availability and user trust. Therefore, enabling fast unlearning is essential to promptly restore a reliable model and minimize the impact of the attack. 
Moreover, the unlearning algorithm itself could be the target of attacks \cite{sheng2024robust, chen2025fedmua}. For example, an attacker could perform a Denial of Service (DoS) attack by repeatedly triggering the unlearning procedure to disrupt and potentially hijack the training process.

To address all these challenges,} we present \textbf{FedUP}, an algorithm to efficiently mitigate the influence of {\color{black} multiple} colluding malicious clients on the global FL model by pruning specific connections within the attacked model. FedUP is designed to start unlearning as soon as a malicious client is detected.
Our approach is practical because the server uses only the weights of the local models sent by the clients to identify which connections to inhibit. To achieve efficiency, FedUP relies exclusively on the clients' weights from the last training round before the unlearning phase. 
%The algorithm computes the difference between benign and malicious client updates to identify the highest magnitude weights that have undergone conflicting modifications. These weights reflect clients' efforts to adjust the model to perform better on their local data.
%By identifying the most dissimilar weights, it becomes possible to determine those that are most in conflict between benign and malicious behavior. By precisely targeting these connections, our approach effectively removes harmful contributions while preserving the model’s integrity and performance. Original model performance is then recovered after a limited number of training rounds. 
{\color{black} Identifying and removing malicious contributions is particularly challenging due to the complex overlap of weight updates across clients and the risk of destabilizing the model. Not all weights contribute equally to the model's behavior, as some have a stronger influence than others. Simply choosing the highest magnitude weights of the malicious clients does not yield satisfactory results, as shown in Section~\ref{sec:ablation}. Our algorithm addresses this by computing the difference between benign and malicious client updates, focusing on high-magnitude weights that exhibit conflicting modifications. These weights are more likely to affect the model's decision boundaries and learning dynamics. Pruning is applied selectively to specific layers, since removing connections across the entire network would compromise stability and generalization. This careful targeting ensures that only the most harmful and influential connections are removed, allowing the algorithm to eliminate malicious influence while preserving benign knowledge and maintaining model performance. FedUP is executed directly by the server once malicious clients are detected. It also incorporates a rate limit mechanism against DoS attacks on the FU procedure.}

Summing up, the main contributions of our work are the following:
\begin{itemize}
{\color{black}
\item We propose an effective server-side-only FU algorithm that recovers a poisoned global model by simultaneously removing one or multiple malicious contributions.
\item FedUP operates under a strong adversarial threat model in which up to $50\% - 1$ of the clients may be malicious. It is also resilient to emerging attacks that target FU, such as DoS attacks.}
\item The solution is lightweight and storage-efficient, relying solely on the model weights collected during the last training round before unlearning.
\item Our algorithm is task-agnostic and can be seamlessly integrated into any FL framework without requiring modifications.
\item We provide technical insights to support optimal fine-tuning of the algorithm's parameters.
\item We demonstrate the effectiveness and efficiency of our solution through extensive experiments {\color{black} involving three different models and two datasets, evaluated under both IID and Non-IID settings.} Comparisons with State-of-the-Art (SOTA) FU approaches show that FedUP consistently achieves successful unlearning while being significantly faster than {\color{black}SOTA} methods.
\end{itemize}

The rest of the paper is organized as follows. Section~\ref{sec:related} section highlights the current state of the art in FU and explains how our solution differs from and improves upon it. {\color{black}Section~\ref{sec:threat} defines the threat model, detailing the attacker’s goals, knowledge, and capabilities.} Section~\ref{sec:fedup} provides a detailed description of the FedUP algorithm. Section~\ref{sec:math} elaborates on key factors and hyperparameters that influence the efficiency and effectiveness of the procedure. Section~\ref{sec:res} showcases the results of our extensive experiments. Finally, Section~\ref{sec:conclusion} summarizes our findings.

% COMMENTI SUL RELATED. COSI' COM'E' NON MI CONVINCE. a cosa ci serve questo related ...per evitare di spargere nel aper in seizoni non appropriate alcune considerazioni e per dare evidenza al nostro lavoro . farei related ampio. CERCHEREI DI RISTRUTTURARE LE SEZIONI SE ABBIAMO SPAZIO: 1. related su attacchi a FL 2. related su tecniche per mitigare gli attacchi in Fl focalizzandoci  su categorie che servono al nostro paper: 1. tecniche in caso di scenari trusted tra cui FU 2. deection?  3. FU come tecnica di mitigazione ma in scenari untrusted 4. DOS per ...riprendere quanto si dice nella sezione threat model citando paper 23. Tabella 1 va messa alla fine di tutto, va spiegata e forse va aggiunto in caso di collusione e dos anche?
\section{Related Work}
\label{sec:related}
{\color{black} Adversaries can launch targeted poisoning attacks by manipulating training to degrade model performance or inject malicious behaviors. There are two common approaches. The first one is label-flipping \cite{jiang2023data,jebreel2024lfighter}, where the attacker intentionally mislabels training data to bias the model’s decision boundary. The second approach is backdoor attacks \cite{xie2019dba,bagdasaryan2020backdoor}, in which a trigger pattern is embedded into training samples so that the model learns to associate this pattern with a target label. This enables the attacker to manipulate predictions at inference time whenever the trigger is present. 
In light of these attacks, this section focuses on related work topics relevant to contextualizing FedUP. We review techniques for attack detection, as FedUP depends on an effective detection mechanism. Additionally, we discuss FU approaches in both trusted settings, where unlearning is collaborative, and untrusted environments, where FU acts as a mitigation against adversarial contributions. Finally, we consider attacks targeting FU  itself.

\subsection{Detection Techniques}
\label{sec:detect}
Detecting malicious clients in FL involves differentiating adversarial updates from legitimate ones to safeguard the global model. Broadly, detection strategies fall into two main groups \cite{huang2024federated,mu2024feddmc}: those that leverage clean validation datasets and those based on clustering or analyzing the similarity of model updates. 
Some approaches, such as FLTrust \cite{cao2020fltrust}, maintain a trusted dataset on the server side to evaluate client updates. Clients whose updates closely match the server’s model receive higher trust scores, which helps filter out malicious contributions. Similarly, Li et al. \cite{li2020learning} use variational autoencoders trained on clean validation samples to detect anomalous updates. Despite their effectiveness, these methods require access to clean and representative data, a condition often difficult to meet in real FL scenarios. Other methods focus on grouping or measuring the similarity between client updates. 
%Multi-Krum \cite{blanchard2017machine} selects a subset of updates that are mutually similar to build the global model, reducing the influence of outliers. 
DnC \cite{fung2020limitations} applies dimensionality reduction followed by spectral analysis to detect malicious clients but requires prior knowledge of the number of adversaries. Auror \cite{shen2016auror} uses K-means clustering, while FoolsGold \cite{fung2020limitations} identifies attacks based on the similarity and diversity of client contributions. In \cite{gu2021detecting}, the authors propose FedCVAE, an unsupervised framework that utilizes a conditional variational autoencoder on the server side. By leveraging reconstruction errors as anomaly scores, FedCVAE dynamically filters out malicious updates. FedPRC \cite{ma2022personalized} combines clustering with density-based anomaly detection to exclude malicious updates effectively. 
FLDetector \cite{zhang2022fldetector} identifies malicious clients by examining the consistency of their model updates. A limitation of this method is its high computational cost. LFighter \cite{jebreel2024lfighter} is a detection method that analyzes gradient patterns in the output layer, exploiting discrepancies between malicious and honest updates. LFighter clusters and filters updates before aggregation, showing strong performance across different data distributions and model sizes. Finally, FedDMC \cite{mu2024feddmc} reduces dimensionality using PCA, employs binary tree clustering to mitigate noise, and incorporates a self-ensemble module to enhance detection robustness.
{\color{black}
Existing detection techniques become ineffective when malicious clients constitute the majority, as they rely on a benign majority to identify adversarial updates reliably \cite{li2020learning,zhang2022fldetector,huang2024federated}. Moreover, in such scenarios, even existing robust aggregation methods may struggle to converge \cite{cao2021provably, le2023privacy, mu2024feddmc}. 
Therefore, FedUP also assumes a benign majority and is designed to operate effectively under this assumption.}
\subsection{FU in Benign Environments}
Most existing works on FU are developed under benign settings, assuming honest participation from all clients. While our focus is primarily on unlearning in adversarial contexts, we briefly mention two notable approaches in benign scenarios.
For a more exhaustive overview of FU in benign settings, we refer the reader to the survey \cite{romandini2024federated}.
The first solution is FedEraser~\cite{9521274}. It achieves client unlearning by leveraging historical parameter updates stored at the server to speed up the retraining process. Specifically, FedEraser conducts a calibration phase that iteratively sanitizes updates by only including the retained clients. Each round involves local training with calibrated updates and server-side aggregation, producing a sanitized global model to be broadcast in the next re-calibration round, continuing until all past updates are recovered. In addition to the need to store numerous client updates, the algorithm involves a recovery process that can require up to as many rounds as the training process itself. Unlike FedEraser, FedUP only requires the weights from the last round before unlearning and involves a recovery process that needs only a few additional rounds.

The second work in benign scenarios is \cite{wang2022federated}, in which the authors present a method to forget specific categories in Convolutional Neural Networks (CNN) models within FL. Using Term Frequency-Inverse Document Frequency (TF-IDF) scores, the method quantifies channel discrimination, pruning those highly associated with target categories to achieve unlearning. Fine-tuning follows pruning to restore model performance. This approach is specifically designed for CNNs, which limits its applicability to other model types. Additionally, it requires clients to compute and transmit extra information to the server to identify the most significant channels, making it unsuitable for use in malicious contexts. FedUP relies solely on updates from the most recent training round and is applicable to any neural network.

{\color{black} The first work is used in our experimental evaluation for comparison, as it is one of the few approaches that offer publicly available and functional code. Although the second work proposes a method similar to ours, it does not provide open-source code, which hinders a direct comparison.}

\subsection{FU as a Mitigation Strategy in Malicious Settings}

%This section highlights a small selection of the most significant related work in FU. 
%First, we present works assuming trusted and benign clients, unsuitable for malicious contexts, followed by those addressing untrusted or malicious clients. %spegare perche il reated su trusted clients è utile...per confrontare cosa? altrimenti ci si puo' chedere perchè viene inserito

Few works specifically address the unique challenges posed by a malicious context.
{\color{black}
In \cite{wu2022toward}, the authors propose a pruning method that removes neurons responsible for triggering misbehavior when backdoor patterns are detected. Although the method is designed for adversarial settings, it requires clients to compute additional information, which malicious clients are typically unwilling to provide or may deliberately falsify. Furthermore, the method is applicable only to specific types of attacks and is limited in that it cannot effectively remove multiple malicious clients simultaneously. In contrast, FedUP requires no additional information beyond the models from the final training round and can be applied across a wide range of scenarios, including non-adversarial contexts, as demonstrated by the experimental results. Moreover, FedUP is capable of removing multiple malicious clients at the same time.}

FedRecover~\cite{10179336} aims to recover a poisoned model attacked by malicious clients using historical information. It employs retained local model updates to estimate the updates that would be produced during a retrain-from-scratch process, using the Cauchy mean theorem and the L-BFGS algorithm \cite{L-BFGS} to approximate the integrated Hessian matrix efficiently. To improve the model's performance, FedRecover implements fine-tuning strategies, including warm-up, periodic correction, abnormality fixing, and final tuning phases. Storing all client updates during training and running the L-BFGS algorithm could be inefficient both in terms of storage and time. In contrast, FedUP only requires storing the weights from the models received in the last round before unlearning. Additionally, the pruning process consists of simply setting the selected weights to zero, thereby ensuring efficiency.
FAST~\cite{guo2023fast} is a protocol designed to mitigate the influence of Byzantine participants on the global model in FL. The server retains all client updates and adjusts the model parameters by subtracting the contributions of malicious clients from the final global model in each training round. The server compares the accuracy of the current unlearning model with the previous one, continuing the unlearning process if the current model outperforms its predecessor. This process is repeated until the maximum number of attempts is reached. To address potential loss in performance, the server uses an additional small benchmark dataset for supplementary training, enhancing the overall accuracy. Similar to previous work, storing all updates is inefficient in terms of storage {\color{black} and raises privacy concerns due to model and gradient inversion attacks.} Furthermore, in many cases, the server may not have access to additional datasets, which limits the applicability of such algorithms. In contrast, FedUP does not require any extra data sources.

{\color{black}
In \cite{wu2024unlearning}, the authors propose a novel approach designed explicitly for backdoor removal in FL systems. Their method combines historical update subtraction and knowledge distillation to eliminate malicious clients' influence while preserving the global model's performance. The approach does not require additional clients' participation rounds and is compatible with various neural network architectures. However, the server requires additional data, which is not always available. In contrast, FedUP does not need supplementary data, making it more suitable for scenarios where access to external data is not possible or desirable.
Lastly, in \cite{khalil2025not} the authors propose NoT, a FU method that removes a participant’s contribution from a model without requiring access to the original data or additional storage. The approach leverages weight negation to disrupt the learned parameters while retaining the model’s capacity to recover. Although NoT is not explicitly designed for removing malicious updates, the authors demonstrate its effectiveness against backdoor attacks. However, because the negation is not targeted at specific contributions, it can excessively disrupt the model, resulting in a higher number of recovery rounds, as observed in our comparison in Section~\ref{sec:res}.

}

\begin{table}
    \centering
    \resizebox{\linewidth}{!}{
    \begin{tabular}{lccccc}
    \toprule
         \textbf{Work}&\textbf{\makecell{Malicious\\ Clients}} &\textbf{\makecell{No Historical\\Updates}} &\textbf{\makecell{Multiple\\Unlearning}} &\textbf{\makecell{No Extra\\Data}} 
         &\textbf{\makecell{DoS\\Resilient}}\\
         \midrule
         \cite{wang2022federated}& \xmark  & \cmark & \xmark & \xmark & \xmark\\
         \cite{9521274}& \xmark & \xmark & \xmark & \cmark & \xmark\\
         \cite{wu2022toward}& \cmark & \xmark& \xmark&\cmark& \xmark\\
         \cite{10179336}& \cmark& \xmark& \xmark&\cmark& \xmark\\
         \cite{guo2023fast}& \cmark & \xmark& \xmark&\cmark& \xmark\\
         \cite{wu2024unlearning}& \cmark & \xmark& \cmark&\xmark& \xmark\\
         \cite{khalil2025not}& \cmark & \cmark& \xmark&\cmark& \xmark\\
         \textbf{Ours} & \cmark& \cmark & \cmark&\cmark&\cmark\\
         \bottomrule
    \end{tabular}
    }
    \caption{Comparison with related work.}
    \label{tab:related}
\end{table}

{\color{black} No described method provides open-source code, but we include NoT in our comparison since its simple design enabled us to reimplement it for our experimental evaluation.} Table~\ref{tab:related} summarizes the main novelty of FedUP compared to the most relevant related work. The columns highlight key features, including the ability to handle malicious clients, avoid storing historical updates, unlearn multiple contributions simultaneously, operate without requiring extra data beyond the models, and resist DoS attacks. To the best of our knowledge, FedUP is the first algorithm that enables the simultaneous removal of contributions from one or more malicious clients, requiring only the models from the last training round before unlearning, and is inherently resistant to DoS attacks by design.

\subsection{Attacks Targeting FU}

Recent studies have investigated how unlearning mechanisms can be maliciously exploited \cite{sheng2024robust,chen2025fedmua}. In these contexts, malicious clients may aim to trigger the unlearning algorithm to disrupt the global model and cause it to forget valuable knowledge. Alternatively, malicious clients might try to continuously activate the unlearning algorithm to hijack the training process, effectively launching a DoS attack that prevents stable model convergence. Most of the works in the literature still lack protection mechanisms against these threats. {\color{black} FedUP is the first work to present a mechanism for preventing DoS attacks on the FU process.}
}
\begin{comment}
\textbf{premessa su attacchi}

alla luce di questo questo related si concentra su aspetti rilevanti per contestualizzre fedup. 

tecniche per detectare gli attacchi (fedup richiede un algoritmo di detection), FU (contesto trusted), FU come mitigazione (contesto untrusted)

-> considerzione finale: possibili attacchi ad fu
(aggiungere a tabella quali lavori sono resilient to DoS attacks).
\end{comment}

{\color{black}
\section{Threat Model}
\label{sec:threat}

This section describes the threat model considered in our setting and outlines the main assumptions. Therefore, we clarify the adversary’s goals, knowledge, and capabilities. Our threat model is consistent with those presented in related works \cite{10179336,wu2022toward, wu2024unlearning}.% direi qui sinteticamente che ci concentriamo su due scenari di attacc diversi: client malevoli che cìvolgiono alterare il modello e client malevoli che vogliono attaccare unlearning per renderlo inefficace
\subsubsection{Adversary and Adversary’s Goals}
%\textbf{attacker può fare due attacchi: soliti tipici e nuovo dos}
We consider a scenario where a group of malicious {\color{black}and possibly colluding} clients collaborate to compromise the training process in an FL system. These clients can manipulate their local updates, such as gradients or model weights, before sending them to the central server.  We considered targeted poisoning attacks, where the adversary aims to cause specific misclassifications in the global model. 
In our scenario, malicious clients have already participated in multiple training rounds, embedding malicious updates into the global model. {\color{black} The server is assumed to be honest and is equipped with a mechanism to detect malicious behavior during the training process.} Our objective is to neutralize the malicious influence already embedded in the global model without degrading performance on benign tasks. 
Beyond the typical objective of poisoning the model, malicious clients can also perform a DoS attack that targets the unlearning mechanism. In this case, the adversary attempts to repeatedly trigger the unlearning process in order to hijack the training pipeline. This can result in disrupted convergence, excessive resource consumption, and overall system instability.

%As already said, {\color{black} the detection of malicious clients is beyond the scope of this work. Nevertheless,} our approach {\color{black}can be integrated with all existing detection methods and relies on state-of-the-art detection solutions for malicious client detection}\cite{li2020learning,zhang2022fldetector,jiang2023data}. 
%Moreover, recent studies have investigated how unlearning mechanisms can be maliciously exploited \cite{sheng2024robust,chen2025fedmua}. In these contexts, malicious clients may aim to trigger the unlearning algorithm to disrupt the global model and cause it to forget valuable knowledge. Alternatively, malicious clients might try to continuously activate the unlearning algorithm to hijack the training process, effectively launching a denial-of-service (DoS) attack that prevents stable model convergence.

%rebecca: tutta la parte da As already said,  per me qui non ci sta 
%The adversary’s goals can be summarized as follows:

%Goal I: Attack effectiveness. The adversary aims to manipulate the global model so that it behaves incorrectly in a controlled way—either broadly (untargeted) or on specific inputs (targeted).

%Goal II: Model utility preservation. The adversary seeks to maintain the global model’s apparent performance on benign, non-triggered data to avoid detection.

\subsubsection{Adversary’s Knowledge}
We assume that adversaries possess knowledge of the FL process, the training algorithm, and the global model architecture. They are capable of crafting malicious updates tailored to the training objective. However, they do not have access to the local data or updates of other participants. This reflects a realistic threat model where attackers operate under limited information and cannot directly observe or influence the contributions of benign clients.

%We assume a realistic scenario in which the adversary has access to the compromised clients' local training data and model updates.  The adversary knows how to craft poisoned updates and when to inject them, but the global model architecture and other clients’ parameters may remain unknown. This setting reflects a common assumption in prior work, such as the FedRecover framework, and represents a practical threat model for federated learning environments with partially trusted participants.

\subsubsection{Adversary’s Capabilities}
Adversaries are assumed to control up to $50\% - 1$ of the total participants, which aligns with adopted threat models in the literature \cite{jiang2025towards}. These malicious clients can upload manipulated model updates during the training process. However, adversaries cannot control the aggregation algorithm or the server's behavior, which is assumed to be honest and to execute the protocol correctly.

%Given their limited proportion, the adversary must carefully design updates to evade defenses and maximize the impact of their attacks. The effectiveness of these attacks may depend on factors such as the degree of data heterogeneity (IID vs. Non-IID settings) and the number of compromised clients.

}

\begin{figure}
    \centering
\includegraphics[width=0.8\linewidth]{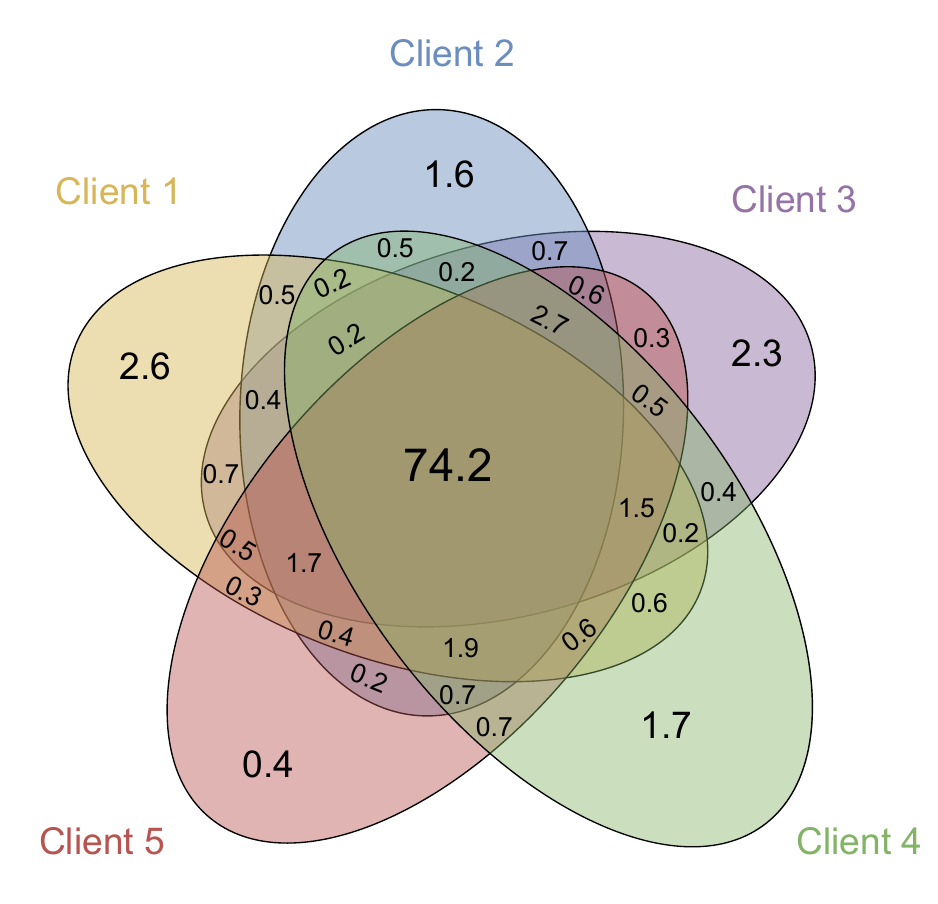}
    \caption{\color{black}Percentage of weights influenced by each client in a deep convolutional layer of MobileNetV2 after one training round on CIFAR-10.}
    \label{fig:venn}
\end{figure}
\begin{figure*}[h!!!]
    \centering
\includegraphics[width=0.9\textwidth]{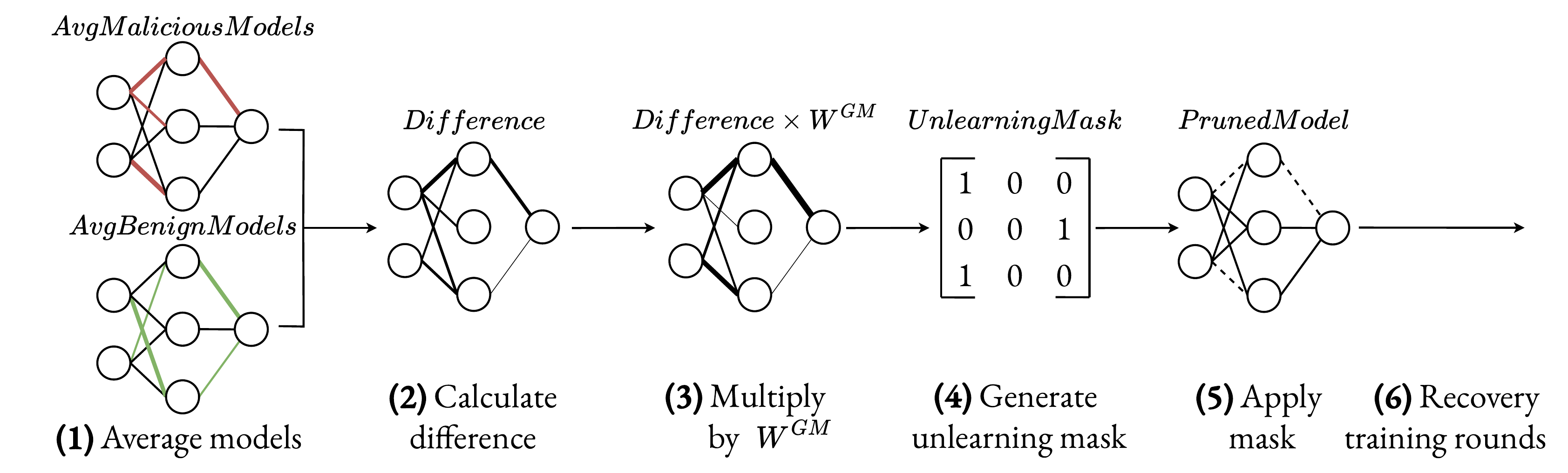}
    \caption{Unlearning procedure after detecting malicious clients. 
    }
    \label{fig:unlearning}
\end{figure*}

\section{FedUP Design}
\label{sec:fedup}
Our algorithm leverages one-shot soft pruning \cite{he2018soft,cai2021softer}, 
which allows the model to zero out specific weights while retaining the potential to retrain and recover them if needed.
For simplicity, from this point, we will refer to soft pruning as pruning. However, unlike traditional pruning techniques that focus on identifying and removing less relevant connections to make the neural network more efficient, our algorithm reverses this principle entirely. Instead of discarding the less critical connections, we eliminate the most important ones that significantly influence the network's output, contributing more to the model's predictive power. 
Specifically, FedUP zeroes out the salient weights identified in the local models sent by malicious clients from the global model. 
{\color{black} However, identifying the most salient weights is a challenging task. Figure~\ref{fig:venn} illustrates how each client influences different subsets of the model’s weights in each training round. As shown, there is a significant overlap among the weights affected by different clients, which complicates the isolation of individual contributions and the accurate identification of malicious updates. 
Nevertheless, } pruning effectively removes any potentially malicious contributions contained in these connections, thereby neutralizing the influence of the malicious clients. 
Let us notice that FedUP can be applied to remove multiple malicious clients' contributions simultaneously. The process remains unchanged regardless of the number of malicious clients. 

Figure~\ref{fig:unlearning} presents the unlearning procedure.
It is important to outline that the FedUP unlearning algorithm does not affect traditional training when not activated. 
The unlearning process is executed by the server only when it detects the presence of one or more malicious clients.
First, the server creates two separate reference models: one by averaging benign models and the other by averaging malicious ones (1).
The step is designed to enhance efficiency, especially when dealing with a large number of clients, by minimizing the need for extensive pairwise comparisons.
Then, it computes the difference between the averaged models to identify the conflicting weights between the benign and malicious clients (2). However, weights in neural networks are not equally important. To account for this, weight differences are scaled by their global model values (3). After that, it selects the highest magnitude weights that differ the most to create an unlearning mask (4). The mask is then applied to the model produced by averaging the benign models (5).
This step allows for the exclusion of the malicious clients' models and their contribution in that round as well. However, some pruned weights might include benign information from non-malicious clients. As a result, pruning can generally lead to a reduction in the model's performance. To recover the performance lost, FedUP incorporates additional training rounds after pruning. Finally, the server sends the pruned global model to the remaining clients to perform recovery training rounds to restore the performance (6).

FedUP seamlessly works whether clients send models, like in FedAVG~\cite{mcmahan2017communication}, or gradients, like in FedSGD, as the information remains consistent in both cases. The most dissimilar weights are the same ones with the most dissimilar gradients.
Moreover, our method does not require saving or receiving any additional information apart from the models or gradients from the clients. This allows for easy integration into any FL framework.
The same goal of FedUP can be achieved through {\color{black}natural} forgetting~\cite{french1999catastrophic}, but at the cost of numerous additional rounds that make the {\color{black}natural} forgetting process inefficient when the model has already converged~\cite{gao2024verifi}.
{\color{black}
Therefore, the underlying intuition of our approach is to strategically perturb the model’s current state by pruning selected connections, thereby guiding it away from configurations influenced by malicious contributions.} This accelerates {\color{black}natural} forgetting and forces the model to reconfigure itself while discarding the data to be forgotten during the recovery rounds.
\begin{figure}[t]
    \centering
        \centering
        \input{histogram} 
        \caption{Normalized squared difference between averaged benign and malicious weights {\color{black} with 20 clients and 6 malicious.}}
    \label{fig:weights_diff}
\end{figure}
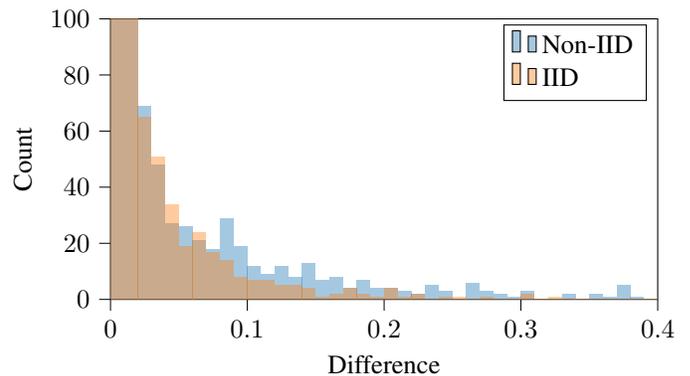
\subsection{Identifying Most Important Connections}
Our objective is to identify the most important weights for a specific client using only the latest updates sent by the clients. {\color{black} These identified weights determine the unlearning mask applied to the global model. Details of the procedure are provided in Algorithm~\ref{alg:detect}. } Many studies on weight importance are designed for centralized contexts \cite{yu2018nisp,molchanov2019importance}. Identifying critical connections for specific clients in FL is challenging because multiple clients can update the same weight simultaneously. The variation in weights between rounds helps to understand the direction in which a client aims to drive the global model \cite{gill2023provfl,jebreel2024lfighter}. By applying this principle to detect which weights are primarily influenced by each group of clients, we compute the squared difference between benign and malicious updates (line 5 of Algorithm~\ref{alg:detect}).
Figure~\ref{fig:weights_diff} illustrates the squared differences between the averaged benign and malicious models in a CNN after one training round under both IID and Non-IID settings. In both cases, malicious clients also possess benign data that makes them more stealthy. As shown, most weights experience minimal changes. 
%In the IID setting, where clients have similar data, the weight differences are minor, with the distribution skewed towards zero and only a few weights showing notable divergence. In contrast, in the Non-IID setting, where even benign clients possess different data distributions, the weight differences become more pronounced. 
However, malicious clients are more likely to have weights that diverge significantly from benign clients, making their influence more noticeable.
As said, weights in a neural network contribute differently. Those with higher magnitudes typically have a greater influence on the network's behavior \cite{NIPS1989_6c9882bb}. Therefore, the magnitude of the weight being modified must also be considered. For this reason, the difference between the weights is multiplied by the weight value in the global model (line 6 of Algorithm~\ref{alg:detect}). This adjustment prioritizes the most significant weights both for the client to unlearn and for the model itself, making the pruning process more efficient and effective.

\subsection{Unlearning Procedure}
This procedure can be executed as soon as the server detects the presence of one or more malicious clients. This can also happen multiple times during the FL process. 
First, the server calculates the unlearning mask by selecting each layer's most important weights. The steps are presented in Algorithm~\ref{alg:detect}. The percentage of weights $\mathcal{P}$ to be pruned could be estimated based on the complexity of the model and the data distribution. More details are discussed in Section~\ref{sec:thr}.
Then, the server uses the unlearning mask to zero out the weights in the model produced by averaging the benign models (line 8 of Algorithm~\ref{alg:masking}). As said, applying the mask on this model already removes the contribution sent by malicious clients in the last training round. 
Algorithm~\ref{alg:masking} shows the procedure to apply the mask.
Another important consideration is regarding the layers on which the mask is applied. Not all layers have the same characteristics, and indiscriminately inhibiting weights in every layer can significantly degrade the network's performance. For this reason, pruning is limited to fully connected and convolutional layers (line 7 of Algorithm~\ref{alg:masking}), as seen in other works \cite{han2015learning,frankle2018lottery,wu2021differential}. 
The result of this operation is a network that has become more or less sparse, depending on the percentage of the pruned weights. In centralized contexts, it has already been shown that sparsity can facilitate data unlearning \cite{liu2024model}. At this stage, the network has lost part of its prediction capabilities and shifted away from the potential local minimum it had reached after convergence.

\begin{algorithm}[tb]
\caption{Generation of the  Unlearning Mask}
\label{alg:detect}
\textbf{Input}: $localModels^t, globalModel^{t-1},  \mathcal{P}$\\
\textbf{Output}: $mask$
\begin{algorithmic}[1] %[1] enables line numbers
\STATE $maliciousModels \gets [localModel_i \mid i$ \textbf{is} $malicious]$
\STATE $avgMaliciousModels \gets Avg(maliciousModels) $
\STATE $benignModels \gets [localModel_i \mid i$  \textbf{is not} $malicious]$
\STATE $avgBenignClients \gets Avg(benignModels) $
\STATE $difference \gets (avgMaliciousModels - avgBenignModels)^2$
\STATE $rank \gets difference \times globalModel^{t-1}$
\STATE $mask \gets [$ $]$
\FOR{$i = 0$ \TO $len(rank)$}
\STATE $rankSorted \gets sortDescending(rank[i])$
\STATE $weights \gets choose(rankSorted, \mathcal{P})$
\STATE $mask.append(weights)$
\ENDFOR
\STATE \textbf{return} $mask$
\end{algorithmic}
\end{algorithm}

\begin{algorithm}[tb]
\caption{Apply Mask on Global Model}
\label{alg:masking}
\textbf{Input}: $mask, avgBenignClients$\\
\textbf{Output}: $prunedGlobalModel$
\begin{algorithmic}[1] %[1] enables line numbers
\STATE $prunedGlobalModel \gets avgBenignClients.copy()$
\STATE $layers \gets prunedGlobalModel.getLayers()$
\FOR{$i = 0$ \TO $len(layers)$}
    \IF{$layer[i] \in [Dense,Convolutional]$}
        \STATE $layer[i][mask[i]] \gets 0$
    \ENDIF
\ENDFOR
\RETURN $prunedGlobalModel$
\end{algorithmic}
\end{algorithm}
\subsection{Restoring Performance}
Pruning is inherently destructive, resulting in the loss of information stored in the zeroed out connections. Each connection stores information from multiple clients in a distributed training context such as FL. Therefore, pruning the weights, even if related to malicious clients, will likely erase some of the information learned from other clients. 
To mitigate this effect, a few additional {\color{black} FL} training rounds performed by the remaining benign clients are sufficient to restore the performance lost during the unlearning procedure. The clients use the same local dataset and configuration as in the standard FL training process, and these recovery rounds remain transparent to them. This operation further strengthens the removal of the excluded client's contribution, as the model will be more finely tuned to the remaining clients. 
Estimating the number of recovery rounds is a complex task, as it is equivalent to estimating the number of rounds a model needs to converge in a standard FL process. However, in Section~\ref{sec:rec}, we provide some considerations that offer insight into the estimation.
{\color{black}
\subsection{Unlearning Rate Limit}
\label{sec:fedup:rate}
To prevent DoS attacks as described in  Section~\ref{sec:related} and Section~\ref{sec:threat}, FedUP enforces a rate limit on the unlearning process. In FedUP, unlearning is never triggered directly by clients but is instead initiated by the server upon detecting new malicious participants. To manage this process, the server defines an unlearning rate threshold $\mathcal{T}$, which sets the minimum number of training rounds that must pass between consecutive unlearning operations. This mechanism ensures that unlearning is not triggered too frequently, protecting the system from repeated and costly recovery operations that could be induced by adversarial activity. FedUP is designed to unlearn the influence of multiple clients simultaneously. If several malicious clients are detected during the period in which unlearning is temporarily disabled due to rate limiting, the server accumulates these detections. Once the threshold $\mathcal{T}$ is reached, FedUP performs a single unlearning step that prunes all identified malicious contributions at once. The threshold can be configured by the system owner to balance responsiveness with system stability, depending on the deployment context and threat landscape. Empirically, as shown in the experiments, the unlearned model tends to restore its original performance within a few rounds following the unlearning phase. Therefore, even a low threshold, such as $T = 10$ rounds, can provide an effective trade-off between mitigating DoS attacks and ensuring timely model recovery. }

\section{Technical Insights}
\label{sec:math}
In this section, we give technical insights and practical guidelines to tune the algorithm properly. 
\subsection{Pruning Most Dissimilar Weights}
\label{sec:prun}
The choice to prune the most dissimilar weights is supported by mathematical evidence, as detailed in the following section.
In FL, each client starts a training round with the global model from the previous round, $ W^{\text{GM}, t-1} $. The client trains this model on its local data, computing updates $ \Delta W^{C_i} $, which adjust the global model to the client’s data. The updated client weights are:  
\begin{equation}
W^{C_i, t} = W^{\text{GM}, t-1} + \Delta W^{C_i}.
\end{equation}
The server aggregates these updates, typically using weighted averaging, to create the new global model for the next round. Here, $ \Delta W^{C_i} $ reflects the client-specific contributions. In a converged model, $ W^{\text{GM}, t-1} $ is near-optimal for the global data distribution, minimizing global loss. For weights already optimized globally, updates $ \Delta w_j^{C_i} $ are small. In contrast, updates for weights critical to reducing a client’s local loss, $ \Delta w_j^{C_i} $, are larger.  At convergence, the difference between benign and malicious weights stabilizes, manifesting two distinct patterns:
\begin{enumerate}
    \item Both parties are satisfied with the weight $w_j$ value, so the updates are nearly zero. These weights usually capture the common knowledge shared by all clients. Malicious clients, for example, may possess benign data to complicate their identification. In this case, the difference is minimal.
    \item Both parties have updated the weight $w_j$ in a conflicting manner. These weights are the contested ones, used by clients to shape the network and minimize the loss on their data. In this case, the difference is more significant.
\end{enumerate}
During training, clients will naturally adjust the weights to align with each other's updates, promoting consensus on those parameters. However, when clients possess malicious data, their update will diverge on a particular weight \cite{jebreel2024lfighter}. This divergence can persist at convergence, given that malicious clients are typically in the minority, as most of the model's knowledge reflects benign clients' contributions.
To identify the weights most targeted by malicious clients, the server computes the differences, $ \Delta w_j^{\text{diff}} $, between all the weights from benign and malicious clients.
Weights with the largest values of $ \Delta w_j^{\text{diff}} $ are the most likely to have been influenced by malicious clients. 

\subsection{Threshold Choice}
\label{sec:thr}
{\color{black}
One of the most crucial aspects for precise and efficient unlearning is choosing an appropriate percentage of weights to prune $\mathcal{P}$. Selecting a value that is too low risks leaving malicious knowledge intact within the model. Conversely, a value that is too high risks pruning too many weights, thereby completely deteriorating the model.
However, determining the optimal pruning percentage is challenging. In most existing works \cite{liu2018rethinking,molchanov2019importance, blalock2020state}, the pruning percentage is typically chosen empirically. To enable a more informed selection, the pruning percentage can be fine-tuned based on the following considerations:
\begin{enumerate} \item \textbf{Weights Separability Between Malicious and Benign Clients ($\mathcal{S}$):} When the divergence between the contributions of benign and malicious clients is more evident, it becomes easier to identify and isolate harmful weights. In this case, slightly more aggressive pruning can be applied without compromising useful knowledge. When separability is low, there is more overlap between benign and malicious updates, making it difficult to prune malicious influence without also affecting benign knowledge. This requires more conservative pruning.
\item \textbf{Ease of Performance Recovery ($\mathcal{R}$): } Another important factor is how quickly the model can regain its performance after pruning. In scenarios where the model architecture and data distribution support efficient relearning, slightly more aggressive pruning can be tolerated since the model will recover in just a few rounds. Conversely, in more fragile setups where performance recovery is slower, pruning should be more conservative to avoid prolonged degradation.
\item \textbf{Natural Forgetting Efficacy ($\mathcal{E}$):} Neural networks naturally forget knowledge that is no longer reinforced through training. Excluding clients from the training process can help accelerate this forgetting. However, the speed remains relatively slow. Moreover, the effectiveness of this mechanism can vary significantly depending on the specific training dynamics. Factors such as model architecture, data distribution, and optimization strategy can all influence how quickly outdated or malicious information is overwritten. When natural forgetting is weak, more explicit pruning is necessary to effectively remove harmful contributions. Conversely, if natural forgetting is efficient, the model can reduce harmful influence with less pruning.
\end{enumerate}
Assuming $\mathcal{S}, \mathcal{R}, \mathcal{E} \in [0, 1]$, where higher values represent greater separability, easier performance recovery, and stronger natural forgetting efficacy, respectively, we propose the following formula to guide the selection of the pruning percentage:

\begin{equation}
\label{eq:P}
    \mathcal{P} \approx \lambda_1 \cdot \mathcal{S} + \lambda_2 \cdot \mathcal{R} + \lambda_3 \cdot (1 - \mathcal{E})
\end{equation}

where $\lambda_1 + \lambda_2 + \lambda_3 = 1$ are weighting coefficients used to balance the influence of each factor, which can be chosen based on empirical observations or specific application needs. 
Although directly estimating $\mathcal{S}$, $\mathcal{R}$, and $\mathcal{E}$ may be challenging, the pruning percentage can be approximated using the degree of data similarity among benign clients, based on the following insights.

In fact, the separability between benign and malicious updates tends to be higher in IID scenarios compared to Non-IID ones. This is because in IID settings, clients’ data distributions are similar, so their model updates follow consistent patterns. Malicious updates therefore stand out more clearly as outliers or divergent patterns, making it easier to identify and isolate harmful weights. Conversely, in Non-IID scenarios, clients hold more diverse and heterogeneous data, leading to greater natural variability in their updates. This overlap between benign and malicious contributions reduces separability and complicates the pruning process.

Moreover, the model's ability to recover performance after pruning also differs between IID and Non-IID settings. In IID scenarios, the consistent and redundant nature of the data across clients allows the model to quickly relearn any slightly pruned benign knowledge, facilitating faster performance recovery. On the other hand, in Non-IID settings, data heterogeneity makes it more difficult for the model to recover lost information, as the removed knowledge might be unique to a specific client or data subset. This reinforces the need for more cautious pruning in such cases, as recovering lost benign contributions is inherently more challenging.

Finally, regarding natural forgetting, its efficacy is typically lower in IID environments. Since the global model continuously encounters similar updates from all clients, it reinforces learned parameters consistently, making it harder for the model to forget outdated or maliciously injected knowledge. In contrast, in Non-IID settings, the global model is exposed to more varied updates, which encourages parameter changes that can overwrite or diminish the impact of previous harmful contributions. This diversity acts as a natural regularizer, accelerating the forgetting of adversarial influence. Our ablation study in Section~\ref{sec:ablation} confirms this behavior, showing that natural forgetting is more effective in Non-IID training.

All these considerations suggest that the data similarity can provide an approximate guideline for selecting the pruning percentage. Therefore, Equation~\ref{eq:P} can be rewritten as:
\begin{equation}
\label{eq:P_two}
    \mathcal{P} \propto  \text{SIM}
\end{equation}
where $\text{SIM} \in [0,1]$ quantifies the degree of data similarity among clients, with higher values indicating more IID-like distributions and lower values corresponding to more Non-IID settings. 
%The coefficient $\alpha$ controls the sensitivity of pruning and can be adjusted based on empirical observations or the desired level of unlearning aggressiveness.
%Empirically, a pruning percentage of  $\mathcal{P} = 10\%$ for IID cases and $\mathcal{P} = 5\%$ for Non-IID cases has proven to be a good trade-off between the efficiency and effectiveness of the algorithm.
While many similarity metrics exist, in this work we choose cosine similarity on the last layer of the model, as it is one of the most commonly adopted approaches in FL literature. To improve separation and control, we normalize the similarity into a variable $ z \in [0,1] $ through the following mapping:
\begin{equation}
z = \frac{\text{SIM}_{\text{orig}} - \text{SIM}_{\min}}{\text{SIM}_{\max} - \text{SIM}_{\min}}
\end{equation}
where $\text{SIM}_{\min}$ and $\text{SIM}_{\max}$ define the range of the original similarity metric. Since we calculate similarity at convergence, we assume the original cosine similarity values fall within $[0.5, 1]$. This normalization helps standardize the similarity values, improving differentiation between Non-IID and IID cases.

To further improve separation, especially between low-to-medium and high similarity values, we introduce a nonlinear mapping for the pruning percentage $\mathcal{P}$. This nonlinear function produces small and relatively flat pruning values for low and medium similarity, capturing the Non-IID nature of the data, while steeply increasing pruning for very high similarity values corresponding to near-IID settings. We define the pruning percentage as
\begin{equation}
    \mathcal{P} \approx (\mathcal{P}_{\text{max}} - \mathcal{P}_{\text{min}}) \cdot z^{\gamma} + \mathcal{P}_{\text{min}},
\end{equation}
where $\mathcal{P}_{\text{max}}$ and $\mathcal{P}_{\text{min}}$ denote, respectively, the maximum and minimum pruning rates.  
The exponent $\gamma$ controls the steepness of the curve: larger values keep the curve relatively flat for low‑to‑mid similarity scores and make it rise sharply as $z \to 1$. This behavior yields finer separation in non‑IID scenarios while allowing more aggressive pruning when the data are highly similar (IID‑like). The parameters $\mathcal{P}_{\text{max}}$, $\mathcal{P}_{\text{min}}$, and $\gamma$ can be tuned to balance pruning aggressiveness and model stability for a given setting.  
In our experiments, we set $\mathcal{P}_{\text{max}} = 0.15$, $\mathcal{P}_{\text{min}} = 0.01$, and $\gamma = 5$, which we found to provide a good trade‑off between sensitivity and robustness. Figure~\ref{fig:pruning_curve} illustrates the nonlinear mapping from the normalized similarity $z$ to the pruning percentage $\mathcal{P}(z)$, with shaded regions indicating typical similarity ranges for IID and Non-IID data.

In our experiments, we calculated mean cosine similarities of 0.99 for IID data and 0.89 for Non-IID data, which normalize to $ z \approx 0.98 $ and $ z \approx 0.78 $, respectively. These values produce pruning percentages $\mathcal{P}$ of approximately 10\% for IID and 5\% for Non-IID cases, effectively adapting pruning aggressiveness based on data similarity.
\begin{figure}
\centering
\begin{tikzpicture}
\begin{axis}[
    width=0.9\linewidth,
    height=5cm,
    xlabel={$z$},
    ylabel={$\mathcal{P}(z)$},
    ymin=0,
    ymax=0.16,
    xmin=0,
    xmax=1,
    grid=both,
    domain=0:1,
    samples=100,
    thick,
    legend style={at={(0.47,0.9)}, anchor=north, legend columns=-1},
    scaled ticks=false, tick label style={/pgf/number format/fixed}
]

% Main curve
\addplot[color=blue, thick] {(0.15 - 0.01)*x^5 + 0.01};
\addlegendentry{$\mathcal{P}(z) = (\mathcal{P}_{\max} - \mathcal{P}_{\min})\, z^\gamma + \mathcal{P}_{\min}$}

\addplot [
    name path=A,
    draw=none,
    domain=0:0.9,
] {(0.15 - 0.01)*x^5 + 0.01};

\path[name path=B] (axis cs:0,0) -- (axis cs:0.9,0);

\addplot [
    blue!10,
] fill between [
    of=A and B,
    soft clip={domain=0:0.9},
];

% IID shaded region (z >= 0.8)
\addplot [
    name path=C,
    draw=none,
    domain=0.9:1,
] {(0.15 - 0.01)*x^5 + 0.01};

\path[name path=D] (axis cs:0.9,0) -- (axis cs:1,0);

\addplot [
    green!10,
] fill between [
    of=C and D,
    soft clip={domain=0.8:1},
];

\node at (axis cs:0.75,0.01) { \textbf{Non-IID}};
\node at (axis cs:0.95,0.01) { \textbf{IID}};

\end{axis}
\end{tikzpicture}
\caption{Pruning rate $\mathcal{P}$ as a function of normalized similarity $z$ with $\gamma = 5$, $\mathcal{P}_{\min} = 0.01$, and $\mathcal{P}_{\max} = 0.15$. The shaded regions highlight typical ranges for Non-IID and IID  scenarios.}
\label{fig:pruning_curve}
\end{figure}
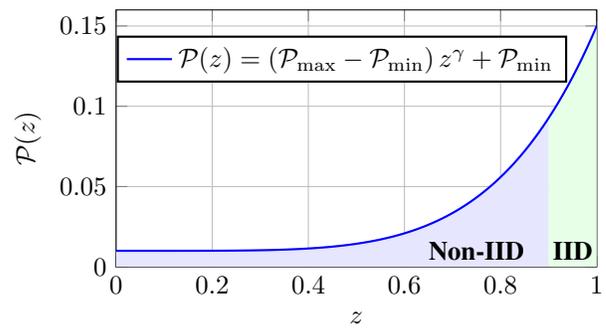

}
\begin{comment}
\begin{enumerate} \item \textbf{Number of weights:} In larger networks, there is more room for all clients’ contributions, which makes it easier to detect targeted weights. This reduces the need for aggressive pruning. Conversely, in smaller networks, client contributions are more likely to overlap, necessitating more pruning to remove malicious clients' influence effectively.
\item \textbf{Skewness of the distances between benign and malicious weights:} High right skewness in the distribution of weight differences suggests that many weights are similar, indicating more overlap between benign and malicious contributions. This scenario is often associated with an IID setup, where more pruning is needed. On the other hand, low skewness indicates dissimilar weights, with less overlap and potentially a Non-IID setup, requiring less pruning. These scenarios have already been shown in Figure~\ref{fig:weights_diff}.
\end{enumerate}
{\color{black}
Considering these factors, the percentage value could be approximated by:
\begin{equation}
    \mathcal{P} \propto \mathcal{N} \left( \frac{1}{W} + Skew \right)
\end{equation}
where $ W $ is the total number of weights in the network, $ \text{Skew} $ represents the skewness, and $\mathcal{N}$ denotes the min–max normalization function that ensures that $\mathcal{P} \in [0,1]$ .}

\end{comment}
\subsection{Number of Recovery Rounds}
\label{sec:rec}

The recovery process involves retraining the model to restore performance after zeroing out specific weights. While predicting the exact number of recovery rounds is challenging, the upper bound is given by the number of rounds needed for retraining from scratch, denoted as $ \mathcal{R}^* $. We hypothesize that an optimal pruning percentage, $\mathcal{P}_{opt}$, exists, which effectively removes the influence of malicious data and ensures $\mathcal{R}_{\mathcal{P}_{opt}} < \mathcal{R}^*$.
Not all weights contribute equally to predicting a sample, implying a subset of weights, $ W_{malicious} $, is responsible for predicting malicious samples, $ \mathcal{D}_{malicious} $. We can determine $\mathcal{P}_{opt}$ by pruning one weight at a time, progressively increasing the pruning percentage with each iteration. As we prune more weights, the model's performance will degrade, increasing the number of recovery rounds, $\mathcal{R}$.  Weights are pruned until $ \mathcal{D}_{malicious} $ is no longer predictable. At $\mathcal{P}_{opt}$, all malicious weights are pruned, leaving only the benign ones, which still retain useful knowledge. As a result, the model requires fewer recovery rounds than retraining from scratch because it already possesses some knowledge.
{\color{black}
To estimate the number of recovery rounds required after pruning, we propose an upper bound that accounts for the extent of pruning. Specifically, the bound is defined as:
\begin{equation}
    \mathcal{R}_{\mathcal{P}} \leq \lceil \mathcal{R}^* \times \mathcal{P} \rceil
\end{equation}
where $ \mathcal{R}_{\mathcal{P}} $ denotes the number of recovery rounds, $ \mathcal{R}^* $ is the number of rounds required to retrain the model from scratch, and $ \mathcal{P} $ is the pruning percentage.
It is important to note that $ \mathcal{R}^* $ is typically higher in Non-IID scenarios due to the difficulty of aggregating heterogeneous client updates. At the same time, according to Equation~\ref{eq:P_two}, IID settings require higher pruning percentages $ \mathcal{P} $. Conversely, Non-IID scenarios tend to necessitate lower $ \mathcal{P} $ values, reflecting the greater risk of mistakenly pruning useful information. These opposing trends naturally balance: when $ \mathcal{R}^* $ is high, $ \mathcal{P} $ is low, and vice versa. This makes $ \lceil \mathcal{R}^* \times \mathcal{P} \rceil $ a practical and robust upper bound for estimating recovery rounds across data distributions.

\begin{comment}
%, and $ \lambda_b $ is the benign loss ratio, defined as the fraction of weights originally influenced by benign clients that are removed during pruning:
\begin{equation}
    \lambda_b = \frac{W_{\text{benign}}^{\text{pruned}}}{W_{\text{benign}}^{\text{total}}}
\end{equation}
\end{comment}
%The pruning percentage $ \mathcal{P} $ alone can provide a coarse estimate of recovery needs.
\begin{comment}
, it does not fully capture the fact that many weights in the model are influenced by both benign and malicious clients. As a result, aggressive pruning may inadvertently remove valuable benign information, causing recovery to be slower and less efficient. Incorporating $ \lambda_b $ into the bound allows for a more accurate estimation by explicitly accounting for this unintended loss of benign contributions.
\end{comment}
}

% rebecca: se abbiamo un related e threat model fatto bene qui dobbiamo mettere anche teroemi su come garantiamo contro fliepped e backdoor in caso di anche colluding malicious client
{\color{black}
\subsection{Security Analysis}
In this section, we evaluate the security guarantees of FedUP under the adversarial model defined in Section~\ref{sec:threat}.
\begin{comment}
\begin{assumption}
\label{assumption:no-work} 
Existing detection mechanisms assume that the majority of clients are benign \cite{li2020learning, gu2021detecting, zhang2022fldetector, jebreel2024lfighter}.
\end{assumption}
\end{comment}
\begin{assumption}
\label{assumption:benign-majority}
%\textbf{alternative version}
In an FL system, the majority of the clients are benign. The current detection solutions (see Section~\ref{sec:detect}) cannot work if the majority of clients are malicious, because it is not possible to separate their updates from those of benign clients in the weight space.

\begin{comment}

That is, for a given round $t$,
\[
\exists\, \epsilon > 0 \quad \text{s.t.} \quad \|\Delta W^{\mathcal{M}}_t - \mathbb{E}_{i \in \mathcal{B}}[\Delta W^i_t]\| > \epsilon,
\]
where $\Delta W^{\mathcal{M}}_t$ denotes the average malicious update, and $\Delta W^i_t$ is the update of benign client $i$.
\end{comment}
\end{assumption}
\begin{comment}
\\
\textbf{Assumption 1.}
In FL with a majority ($50\% + 1)$ of benign clients, model updates from malicious clients deviate significantly from the mean of benign updates \cite{li2020learning, gu2021detecting, zhang2022fldetector, jebreel2024lfighter}.
\\
\end{comment}
\begin{assumption}
\label{assumption:secure_comm}
Malicious clients cannot access benign clients’ updates, and all communication between clients and the server is encrypted, ensuring confidentiality and reducing the attack surface.
\end{assumption}
\begin{comment}
\\
\textbf{\textit{Theorem 1.}} FedUP effectively and efficiently removes the influence of multiple malicious clients, up to $50\% - 1$ of the total clients.
\\
\textit{Proof.} 
FedUP employs a pruning-based mechanism that identifies and zeroes out model weights exhibiting the most significant divergence between benign and malicious updates, effectively removing malicious contributions. The rationale behind this approach is explained in Section~\ref{sec:prun} and Assumption 1. An optimal pruning threshold, empirically estimated to be between 5\% and 15\%, allows the global model to remove malicious influence and recover its accuracy within a few rounds, maintaining its generalization ability.
\\
\end{comment}
\begin{theorem}
FedUP effectively removes the influence of multiple malicious clients, as long as the number of malicious clients does not exceed $\left\lfloor \frac{N - 1}{2} \right\rfloor$, where $N$ is the total number of clients.
\end{theorem}

\begin{proof}
Let $\mathcal{C} = \mathcal{B} \cup \mathcal{M}$ be the set of all clients, where $\mathcal{B}$ is the set of benign clients and $\mathcal{M}$ the set of malicious clients. Assume that $|\mathcal{M}| \leq \left\lfloor \frac{N - 1}{2} \right\rfloor$, so that benign clients are in strict majority. Under Assumption~\ref{assumption:benign-majority}, when benign clients are the majority, malicious updates can be detected. At convergence, benign updates tend to be small and mutually consistent (as discussed in Section~\ref{sec:prun}), while malicious clients, attempting to alter the model toward their adversarial goal, produce updates that conflict with the benign trend. FedUP identifies such conflicting weights by computing update dissimilarities and pruning the most deviant ones, which are most likely to be influenced by malicious behavior. This dissimilarity-based pruning mechanism exploits the structural divergence created by a benign majority. Therefore, as long as $|\mathcal{B}| > |\mathcal{M}|$, FedUP can distinguish malicious from benign knowledge and successfully remove the adversarial influence. 
%If $|\mathcal{M}| \geq \frac{N}{2}$, the assumption no longer holds, and malicious clients can dominate, making it fundamentally impossible to discern adversarial updates from honest ones. However, in such a scenario, even robust FL aggregation schemes become ineffective \cite{cao2021provably, le2023privacy, mu2024feddmc}. Consequently, secure learning cannot be guaranteed under any known approach once the number of malicious clients reaches or exceeds half of the participants.
\end{proof}
\begin{theorem}
All benign model updates in FedUP are protected from eavesdropping and tampering.
\end{theorem} 
\begin{proof}
Following Assumption~\ref{assumption:secure_comm}, all communications between clients and the central server are encrypted. By maintaining confidentiality, adversaries are prevented from crafting more effective or targeted attacks.
\end{proof}
\begin{theorem}
FedUP is resistant to DoS attacks, preventing adversarial misuse of the unlearning mechanism.
\end{theorem}
\begin{proof}
FedUP implements a rate-limiting mechanism, presented in Section~\ref{sec:fedup:rate}, to prevent attackers from continuously triggering unlearning, which could degrade the model. 
\end{proof}
}

\section{Performance Evaluation}
\label{sec:res}
In this section, we present the experiments conducted to evaluate the effectiveness and efficiency of our algorithm. This evaluation aims to demonstrate how effectively FedUP removes the contribution of one or more malicious clients and efficiently restores the original performance in as few recovery rounds as possible.  It also presents a comparison with {\color{black} two SOTA solutions} in terms of storage requirements and the number of recovery rounds needed to restore performance.

\subsection{Experimental Setup}

{\color{black}
\subsubsection{Malicious Setting}
To evaluate the effectiveness of our algorithm, we replicate SOTA adversarial settings, including label-flipping and backdoor attacks. Both attack scenarios are implemented using standard FL benchmarks and follow comparable configurations to ensure consistent evaluation. 
Malicious clients inject poisoned updates into the global model during training. Detection and removal occur in a round when malicious behavior is present, consistent with standard detection-based defense assumptions \cite{zhang2022fldetector}.  We consider two percentages of malicious clients: 30\% (with 10 and 20 clients) and $50\% - 1$ (across all client numbers).  We choose these proportions because at least 30\% of malicious clients are required to execute a successful poisoning attack \cite{tolpegin2020data} and to be consistent with our threat model. 
\\
\noindent\textbf{Label-Flipping Attack:} We reproduce the experimental setup outlined in \cite{tolpegin2020data}.  which provides a robust framework for simulating adversarial clients in FL.  We utilize the same standard datasets as the reference settings, namely CIFAR-10 \cite{Krizhevsky09learningmultiple} and FashionMNIST \cite{xiao2017fashion}, which are commonly used in the literature for benchmarking FL algorithms. This enables an assessment of unlearning performance across varying degrees of difficulty. The neural network models used in our experiments are the same as those in the reference setting. They are standard CNNs similar to the LeNet architecture \cite{lecun1998gradient}.
Malicious clients execute a label-flipping attack, an adversarial attack where the attacker deliberately changes the labels of the training data to incorrect values. This attack aims to degrade the performance of the global model by misleading the training process, causing the model to learn incorrect patterns and make erroneous predictions. 
In our experiments, each malicious client swap 10\% of its labels while keeping the remaining data correctly labeled. The labels to flip are chosen following the same criteria as the reference paper, namely, the most wrongly classified labels in a non-malicious scenario \cite{tolpegin2020data}. This approach is designed to make malicious activities more challenging to detect.
\\
\noindent\textbf{Backdoor Attack:} We replicate the setting proposed in \cite{jebreel2023fl}, which employs a ResNet-18 \cite{he2016deep} model trained on the CIFAR-10 dataset. In this setup, each attacker embeds a $3 x 3$ white pixel square in the bottom-right corner of 10\% input images as a trigger, following \cite{jebreel2023fl}. The model learns to misclassify these triggered inputs while maintaining high accuracy on clean data, thereby evaluating the algorithm’s ability to detect and remove subtle malicious contributions.
\\

{\color{black}
\subsubsection{False Positive Scenario}
FedUP relies on a detection mechanism to identify malicious clients. Although these algorithms are becoming increasingly sophisticated, it is important to consider situations in which the detection may produce incorrect results. One such case occurs when a benign client is mistakenly identified as malicious and is excluded from the training process through pruning. This is commonly known as a false positive scenario. 
%A comparable outcome can also be intentionally caused by adversarial clients who manage to deceive the detection mechanism. This scenario is referred to as a "false flag" attack, where benign clients are misclassified and removed. Whether the exclusion results from a false positive or a false flag attack, the consequence is the same. 
A benign client is unjustifiably removed from the training process, which can negatively affect model performance. 
}
\subsubsection{Comparison with SOTA}
{\color{black} To better demonstrate the capabilities of our algorithm, we include a comparison with SOTA solutions in both malicious and benign settings. The evaluation has been extended to the benign scenario primarily due to the limited availability of publicly released source code for methods designed to operate in adversarial contexts. Nonetheless, this comparison offers valuable insights into FedUP's performance in non-adversarial environments. In the malicious setting, we evaluate FedUP against NoT \cite{khalil2025not} under a backdoor attack scenario.
%To better show the capabilities of our algorithm, we also present a comparison with two SOTA solutions, NoT~\cite{khalil2025not} and FedEraser~\cite{9521274}. 
NoT uses weight negation to disrupt the learned parameters to remove a participant’s contribution without requiring access to the original data or additional storage. It has also been successfully tested against backdoor attacks. Although the source code is not publicly available, its simple design allowed us to reimplement it for our \textit{backdoor attack} experimental evaluation.
For the benign setting, we evaluate FedUP against NoT and FedEraser \cite{9521274}.
%, adopting the same experimental setup used in FedEraser, which involves a small custom CNN and the CIFAR-10 dataset split in an IID manner among 10 clients. 
In FedEraser, performing unlearning requires the participation of the server, which retains all historical updates along with information from the remaining clients. However, it cannot operate in a malicious context, and removing multiple clients simultaneously is not supported. %While FedEraser is designed to forget specific benign clients and FedUP focuses on removing contributions from malicious clients, this comparison is essential to evaluate the unlearning effectiveness of FedUP. 
%FedEraser was chosen for this analysis as it is one of the few solutions with publicly available and verifiable code. 
For additional details, please refer to Section~\ref{sec:related}.
}
{\color{black}
\subsection{FL Parameters}
In all experiments, we use the Adam optimizer with a learning rate of 1e-3, a batch size of 32, and FedAVG as the aggregation strategy.
In the malicious settings, local training is performed for 1 epoch in the label-flipping attack scenario and extended to 3 epochs in the backdoor attack scenario, following the reference settings. We vary the number of clients to 5, 10, and 20, which reflects standard experimental setups commonly used in the FU literature \cite{romandini2024federated}. Additionally, we explore both IID and Non-IID data distributions, the latter simulated using a Dirichlet distribution ($\alpha = 1$), a widely adopted approach for representing realistic data heterogeneity across clients.
In the benign setting, we adopt the same experimental configuration as used in FedEraser, which involves a small custom CNN, the CIFAR-10 dataset split in an IID manner among 10 clients, and 5 local training epochs.
}
\subsection{Evaluation Metrics}
To evaluate the efficacy of the unlearning performance, we measure the accuracy on the to-be-forgotten data before and after the unlearning procedure. This metric is a SOTA way to understand how effectively the model has unlearned the specific data \cite{romandini2024federated}. {\color{black} After unlearning, the new model’s performance should be comparable to a baseline model $\mathcal{B}$, which is obtained by fully retraining the model without any exposure to the malicious data.
This represents the standard performance of the model, unaffected by malicious influences.}
To assess the restoration performance, we evaluate the model on the same test dataset before and after the unlearning process. The new model should maintain equivalent performance.
To evaluate efficiency, we measure the number of recovery rounds, $\mathcal{R}_{rec}$, required to restore the global model's performance, compared to the rounds needed for a retrain from scratch, $\mathcal{R}^*$. Lower values indicate faster performance recovery and, consequently, greater efficiency.
\begin{table*}[h!!!]
    \centering
    \resizebox{\textwidth}{!}{%
    \begin{tabular}{lrrrrrrrrrrrrrrrr}
            \toprule
\multirow{3}{*}{\textbf{\makecell{No. Clients\\(Malicious)}}}
    & \multicolumn{8}{c}{\textbf{IID}} & \multicolumn{8}{c}{\textbf{Non-IID}}\\
        \cmidrule(lr){2-9} \cmidrule(lr){10-17}
            & \multicolumn{2}{c}{\textbf{Test Data}} & \multicolumn{3}{c}{\textbf{Malicious Data}} & \multirow{2}{*}{\textbf{$\mathcal{R}^*$}} & \multirow{2}{*}{\textbf{$\mathcal{R}_{rec}$}}  & \multirow{2}{*}{\textbf{$\mathcal{P}_{opt}$}} &  \multicolumn{2}{c}{\textbf{Test Data}} & \multicolumn{3}{c}{\textbf{Malicious Data}} & \multirow{2}{*}{\textbf{$\mathcal{R}^*$}} & \multirow{2}{*}{\textbf{$\mathcal{R}_{rec}$}}   & \multirow{2}{*}{\textbf{$\mathcal{P}_{opt}$}} \\ 
        &  \textbf{Before} & \textbf{After} & $\mathcal{B}$ & \textbf{Before} & \textbf{After} & & & &\textbf{Before} & \textbf{After} & $\mathcal{B}$ &\textbf{Before} & \textbf{After} & & \\
         \midrule
         5 (2)  &   78 & 79 & 4 & 57 & 5 & 24
          & 1 & 10 %90 client B 73 after 75 after test pruning 21
         & 79 & 80&  4 & 44 & 4
          & 32 &  1 & 5 %99 %66  64
          \\
         
         10 (3) & 79  & 77&   3 & 51 & 2 &26 & 1 & 11%89 client b 72 after 71
         &  78 & 78 & 8 & 30 & 6 &  37 &  2 & 4\\ %96 69 71

         10 (4) &  78 & 78 & 4  & 59 & 3& 27 & 1 & 11 %89 client b 73 after 73
         &77& 77 & 10 & 23 & 9 &  43 & 2 & 4\\ %96 68 67

         20 (6) &  80 & 78&  4& 60 & 5 & 42 &2 & 11 %89 client b 73 after 72
          &   79 & 77& 11 & 37  &11 & 75 & 4 & 3\\%97 72 73
         
         20 (9) &  79 & 79&  5& 42 & 4 & 67 & 2& 10 %88 %73 %73
          & 79 & 79 & 7 & 29 & 8 &  66 & 2 & 3\\%98 70 72
         \bottomrule
    \end{tabular}
    }
    \caption{Accuracy (\%) before and after unlearning on test and malicious clients' data, the number of recovery rounds, and the percentage of weights pruned (\%) used with CIFAR-10 under label-flipping attack}
    \label{tab:cifar10}
\end{table*}
\begin{table*}[h!!!]
    \centering
    \resizebox{\textwidth}{!}{%
    \begin{tabular}{lrrrrrrrrrrrrrrrrrr}
            \toprule
\multirow{3}{*}{\textbf{\makecell{No. Clients\\(Malicious)}}}
    & \multicolumn{8}{c}{\textbf{IID}} & \multicolumn{8}{c}{\textbf{Non-IID}}\\
        \cmidrule(lr){2-9} \cmidrule(lr){10-17}
            & \multicolumn{2}{c}{\textbf{Test Data}} & \multicolumn{3}{c}{\textbf{Malicious Data}} & \multirow{2}{*}{\textbf{$\mathcal{R}^*$}} & \multirow{2}{*}{\textbf{$\mathcal{R}_{rec}$}}  & \multirow{2}{*}{\textbf{$\mathcal{P}_{opt}$}} &   \multicolumn{2}{c}{\textbf{Test Data}} & \multicolumn{3}{c}{\textbf{Malicious Data}} & \multirow{2}{*}{\textbf{$\mathcal{R}^*$}} & \multirow{2}{*}{\textbf{$\mathcal{R}_{rec}$}}  & \multirow{2}{*}{\textbf{$\mathcal{P}_{opt}$}} \\ 
        & \textbf{Before} & \textbf{After} & $\mathcal{B}$ & \textbf{Before} & \textbf{After} &  & & &\textbf{Before} & \textbf{After} & $\mathcal{B}$ &\textbf{Before} & \textbf{After} & \\
         \midrule
         5 (2)   & 84 &88  &  4& 73 & 5 & 22 & 1 & 10 &% 83 84
          83 & 82 & 21 & 54 & 20 &  19 & 2 & 6\\%94 79 80
         
         10 (3) & 85 & 87 & 4 & 65 & 3 & 34 & 1 & 10 %83 81
         & 89 & 88 & 9 & 22 &  7& 36 & 1 & 5\\%95 82 82

         10 (4)  & 84 & 88 & 5 & 72 & 3 & 19 & 1 & 10%83 83
           & 86 & 87 & 11 & 36  & 12 & 19 & 1 & 5\\ %95 81 82

         20 (6)  & 86 & 88 & 6 & 58 &  8 & 51 & 1 & 10 & %85 83 82
          90 & 89 & 5 & 16 & 4 & 64 & 2 & 5\\%95 83 82
         
         20 (9)  & 83 & 82 & 5 & 69 & 6 & 35 & 1 & 11 &%85 83 82
           89 & 88 & 5 & 24 & 6 & 41 & 2 & 5\\%95 83 81 
         \bottomrule
    \end{tabular}
    }
    \caption{Accuracy (\%) before and after unlearning on test and malicious clients' data, the number of recovery rounds, and the percentage of weights pruned (\%) used with Fashion-MNIST under label-flipping attack}
    \label{tab:cifar100}
\end{table*}

\begin{table*}
    \centering
    \resizebox{\linewidth}{!}{%
    \begin{tabular}{lrrrrrrrrrrrrrrrrrr}
            \toprule
\multirow{3}{*}{\textbf{\makecell{No. Clients\\(Malicious)}}}
    & \multicolumn{8}{c}{\textbf{IID}} & \multicolumn{8}{c}{\textbf{Non-IID}}\\
        \cmidrule(lr){2-9} \cmidrule(lr){10-17}
            & \multicolumn{2}{c}{\textbf{Test Data}} & \multicolumn{3}{c}{\textbf{Malicious Data}} & \multirow{2}{*}{\textbf{$\mathcal{R}^*$}} & \multirow{2}{*}{\textbf{$\mathcal{R}_{rec}$}}  & \multirow{2}{*}{\textbf{$\mathcal{P}_{opt}$}} &   \multicolumn{2}{c}{\textbf{Test Data}} & \multicolumn{3}{c}{\textbf{Malicious Data}} & \multirow{2}{*}{\textbf{$\mathcal{R}^*$}} & \multirow{2}{*}{\textbf{$\mathcal{R}_{rec}$}}  & \multirow{2}{*}{\textbf{$\mathcal{P}_{opt}$}} \\ 
        & \textbf{Before} & \textbf{After} & $\mathcal{B}$ & \textbf{Before} & \textbf{After} &  & & &\textbf{Before} & \textbf{After} & $\mathcal{B}$ &\textbf{Before} & \textbf{After} & \\
         \midrule
         5 (2) &  75 & 75 &  1 & 88 & 2 & 33 & 2 & 10 
        & 75 & 73 & 4 & 93 & 5 &  40 & 2 & 5\\
        10 (3) & 76  & 75  & 3 & 81 & 4 & 31 & 2 & 10 
        & 75 & 74 & 2 & 82 & 2 & 32  &  2 & 4\\
        10 (4) &  75 & 73 & 1  & 89 &  2 & 26 & 1 & 10 & 74& 73 & 3 & 90 & 3 &  23 & 1 & 5\\
        20 (6) & 75& 73 & 3&83 & 4& 46& 3 & 10 & 73 & 73 & 2 &73 & 3 & 56 & 2 & 3\\
        20 (9) &  75 & 74&  2& 93 & 3 & 37 & 3 & 12 & 73 & 71 & 2 & 90 & 3 &  47 & 2 & 3\\
        \bottomrule
    \end{tabular}
    }
    \caption{Accuracy (\%) before and after unlearning on test and malicious clients' data, the number of recovery rounds, and the percentage of weights pruned (\%) used with CIFAR-10 under backdoor attack}
    \label{tab:backdoor}
\end{table*}
\subsection{Results}
This section presents the results of FedUP under malicious settings and false flag scenarios, and compares its performance with SOTA solutions.
{\color{black}
\subsubsection{Performance under Malicious Setting}
Tables~\ref{tab:cifar10} and~\ref{tab:cifar100} show the results of FedUP with CIFAR-10 and Fashion-MNIST in both IID and Non-IID settings  under label-flipping attack. Table~\ref{tab:backdoor} instead shows the results of FedUP with CIFAR-10 in both IID and Non-IID settings under backdoor attack. These tables present the performance metrics of the global model before and after applying FedUP, specifically the accuracy obtained on the test set and on the malicious data. They also highlight the number of recovery rounds, $\mathcal{R}_{rec}$, required to achieve these results.
%{\color{black} Regarding the label-flipping attack, }
%Table~\ref{tab:cifar10} shows the results of FedUP on the CIFAR-10 dataset.
As shown, FedUP effectively eliminates the influence of malicious data, as evidenced by the significant drop in the accuracy on their data. Meanwhile, the performance on the test set is almost completely restored in nearly all cases. The performance achieved is comparable to that of the baseline.
%Moreover, Table~\ref{tab:cifar100} illustrates the performance of the global model on the Fashion-MNIST dataset. 
%The results are consistent with those observed with the CIFAR-10 dataset. 
The results are consistent with those observed across all settings and datasets used. Specifically, a notable decrease in performance on data from malicious clients is evident, while the performance on the test set remains essentially unchanged. 

Regarding the number of recovery rounds, the recovery process is generally very fast in all cases. However, the Non-IID case typically requires more recovery rounds on average. This is due to the highly heterogeneous data distribution across clients in Non-IID settings, which makes the updates less stable and the contributions more overlapped. As a result, pruning also removes more benign knowledge, necessitating additional iterations to recover the model effectively. Nevertheless, the number of recovery rounds is considerably smaller than those needed to retrain from scratch.
In all cases, the bound suggested in Section~\ref{sec:rec} holds.

The selection of the percentage of weights to prune follows the guidelines presented in Section~\ref{sec:thr}, demonstrating their effectiveness. In fact, the IID settings require higher percentages of pruning to remove more weights in all cases. In contrast, the Non-IID settings demand lower percentages of pruning due to the presence of more conflicting weights and a more efficient natural forgetting.
{\color{black} In the tables, we report the optimal pruning percentage $\mathcal{P}_{opt}$ that best balances FedUP’s performance in each specific scenario to highlight its capabilities at their best. As shown, these values closely align with the general guidelines, and in many cases, they are identical to them. Nevertheless, FedUP is resilient to variations in the pruning percentage. Notably, using a percentage slightly higher or lower than optimal does not significantly affect the overall outcome. In the former case, a little too much benign knowledge might be removed, requiring a few recovery rounds to restore performance. In the latter case, some malicious knowledge may remain, and additional rounds would be needed to ensure it is fully forgotten through natural training dynamics.}
%Regarding the number of weights, the models used in the CIFAR-10 experiments are larger than the one used for Fashion-MNIST. This difference is reflected in the choice of pruning percentages. A larger number of parameters provides more room for weights to differ, leading to less overlap and, consequently, requiring lower percentages for pruning. Conversely, a smaller number of parameters results in a more significant overlap between weights, necessitating higher percentages to effectively remove the conflicting weights.
}
%\subsubsection{Perfomance under Backdoor Attack}

{\color{black}
\subsubsection{False Positive Scenario} We test this scenario in the backdoor attack setting described earlier{\color{black}, ensuring that the removal of a benign client does not result in an even split between benign and malicious clients, which would completely disrupt the training process and violate Assumption~\ref{assumption:benign-majority}. }
Table~\ref{tab:fps} presents the results for a false positive scenario. The forgetting data refers to the data held by the benign client that was mistakenly excluded from the process. The pruning percentages used are the same as in the respective settings of Table~\ref{tab:backdoor}. As shown, pruning the weights associated with a benign client does not completely disrupt the model. On the contrary, performance is recovered after just a few additional training rounds. Notably, the accuracy on the forgetting data, although it decreases, still remains higher than the accuracy on the test data. This indicates that the data is not entirely forgotten, which is expected and desirable, as the information is valid and beneficial. In fact, FedUP prunes a small percentage of weights that exhibit significant deviation from the collective average. In the case of a false positive, the weights of the benign client would, on average, be close to those of the benign cluster. Thus, pruning would predominantly remove only those weights with substantial deviations, which are likely to be less critical to the model’s general performance. As shown, experimental evidence illustrates that the model can recover its utility within a few additional training rounds, even when benign contributions are pruned. Therefore, the impact of accidental benign client removal is limited and recoverable.
\begin{table}[t]
    \centering
    \resizebox{\linewidth}{!}{%
    \begin{tabular}{llrrr}
    \toprule
        & \multirow{2}{*}{\textbf{\makecell{No. Clients\\(Malicious)}}}
          & \textbf{Test Data}  &\textbf{Forgetting Data} &  \multirow{2}{*}{\textbf{$\mathcal{R}_{rec}$}} \\
         & & \textbf{Before}/\textbf{After} & \textbf{Before}/\textbf{After} & \\
         \midrule
         \multirow{2}{*}{\textbf{IID}} & 10 (4) $\rightarrow$ 9 (4)  & 76 / 76  &  99 / 86 & 2 \\
         & 20 (9) $\rightarrow$ 19 (9)   & 75 / 74 & 98 / 83   &3 \\
        \midrule
         \multirow{2}{*}{\textbf{Non-IID}} &
         10 (4) $\rightarrow$ 9 (4)  & 75 / 74  & 91 / 80& 2\\
        &  20 (9) $\rightarrow$ 19 (9)  & 75 / 73 & 87 / 82 & 5\\

         \bottomrule
    \end{tabular}
    }
    \caption{Model accuracy (\%) before and after a pruning benign client contribution in false positive scenario.
}
    \label{tab:fps}
\end{table}

\begin{table}[t]
    \centering
    \resizebox{\linewidth}{!}{%
    \begin{tabular}{llrrrr}
    \toprule
        \multirow{4}{*}{\textbf{\makecell{No. Clients\\(Malicious)}}} & 
         \multirow{4}{*}{\textbf{Work}} & \multicolumn{2}{c}{\textbf{IID}} & \multicolumn{2}{c}{\textbf{Non-IID}} \\
         \cmidrule(lr){3-4} \cmidrule(lr){5-6}
         &&\textbf{Test Data}  & \multirow{2}{*}{\textbf{$\mathcal{R}_{rec}$}} & \textbf{Test Data} & \multirow{2}{*}{\textbf{$\mathcal{R}_{rec}$}} \\
         & & \textbf{Before}/\textbf{After} & &\textbf{Before}/\textbf{After} &\\
         \midrule
         \multirow{2}{*}{5 (2) } & FedUP & 75 / 75  & 2 & 75 / 73  & 2\\
         & NoT & 75 / 75 & 10 & 75 / 71  & 20\\
         \multirow{2}{*}{10 (3) } & FedUP & 76 / 75  & 2 & 75 / 74  & 2\\
         & NoT & 76 / 75  & 5 & 75 / 73  & 20\\
         \multirow{2}{*}{10 (4) } & FedUP & 75 / 75  & 1 & 74 / 73  & 1\\
         & NoT & 75 / 75 & 6 & 74 / 72  & 20\\
         \multirow{2}{*}{20 (8) } & FedUP & 75 / 75 & 2 & 75 / 73 & 2\\
         & NoT & 75 / 75 & 15 & 75 / 72  & 20\\
         \multirow{2}{*}{20 (9) } & FedUP & 75 / 74  & 3 & 71 / 71  & 2\\
         & NoT & 75 / 73  & 10 & 71 / 71  & 12\\
         \bottomrule
    \end{tabular}
    }
    \caption{Comparison with NoT showing the accuracy (\%) on test and forgetting data before and after unlearning, the number of recovery rounds, and the storage needed (MB) under backdoor attack setting.}
    \label{tab:comparison_backdoor}
\end{table}

\begin{table}[t]
    \centering
    \resizebox{\linewidth}{!}{%
    \begin{tabular}{lrrrr}
    \toprule
         \multirow{2}{*}{\textbf{Work}} &\textbf{Test Data} &\textbf{Forgetting Data} & \multirow{2}{*}{\textbf{$\mathcal{R}_{rec}$}} & \multirow{2}{*}{\textbf{\makecell[r]{Storage\\(MB)}}} \\
         & \textbf{Before}/\textbf{After} & \textbf{Before}/\textbf{After} & & \\
         \midrule
         FedEraser & 56 / 56 & 76 / 56 & 20 &  54\\
         NoT & 56 / 50 & 74 / 50 & 20 & 0.2 \\
         FedUP & 56 / 56 & 74 / 56 & 3 & 2.7\\
         \bottomrule
    \end{tabular}
    }
    \caption{Comparison with FedEraser and NoT, reporting test and forgetting accuracy (\%) before and after unlearning, number of recovery rounds, and storage requirements (MB) in a benign setting with 10 clients and IID data.}
    \label{tab:comparison_benign}
\end{table}

\begin{figure*}[t]
\centering
\begin{subfigure}{0.49\textwidth}
% Left (IID)
\begin{subfigure}{0.49\textwidth}
\centering
\begin{tikzpicture}
\begin{axis}[
    title={IID},
    xlabel={Epochs},
    ylabel={Accuracy},
    xmin=0, xmax=10,
    ymin=0, ymax=1,
    grid=major,
    grid style={dashed, gray!50},
    legend style={at={(0.5,0)}, anchor=west, font=\small},
    width=\textwidth,
    height=4.5cm,
    cycle list name=color list,
    table/col sep=comma
]

% Shaded areas
\addplot[name path=test_q1, draw=none, forget plot] table[x=epoch, y=test_q1] {csv/plot_data_iid.csv};
\addplot[name path=test_q3, draw=none, forget plot] table[x=epoch, y=test_q3] {csv/plot_data_iid.csv};
\addplot[red!20, fill opacity=0.5] fill between[of=test_q1 and test_q3];

\addplot[name path=malicious_q1, draw=none, forget plot] table[x=epoch, y=malicious_q1] {csv/plot_data_iid.csv};
\addplot[name path=malicious_q3, draw=none, forget plot] table[x=epoch, y=malicious_q3] {csv/plot_data_iid.csv};
\addplot[blue!20, fill opacity=0.5] fill between[of=malicious_q1 and malicious_q3];

\addplot[name path=baseline_q1, draw=none, forget plot] table[x=epoch, y=baseline_q1] {csv/plot_data_iid.csv};
\addplot[name path=baseline_q3, draw=none, forget plot] table[x=epoch, y=baseline_q3] {csv/plot_data_iid.csv};
\addplot[green!20, fill opacity=0.5] fill between[of=baseline_q1 and baseline_q3];

% Medians
\addplot[red!50!white, thick] table[x=epoch, y=test_median] {csv/plot_data_iid.csv};
%\addlegendentry{Test Data}

\addplot[blue!50!white, thick] table[x=epoch, y=malicious_median] {csv/plot_data_iid.csv};
%\addlegendentry{Malicious Data}

\addplot[green!60!black!50!white, thick] table[x=epoch, y=baseline_median] {csv/plot_data_iid.csv};
%\addlegendentry{$\mathcal{B}$ (Malicious Data)}

\end{axis}
\end{tikzpicture}
\end{subfigure}
\hfill
% Right (Non-IID)
\begin{subfigure}{0.49\textwidth}
\centering
\begin{tikzpicture}
\begin{axis}[
    title={Non-IID},
    xlabel={Epochs},
    xmin=0, xmax=10,
    ymin=0, ymax=1,
    grid=major,
    grid style={dashed, gray!50},
        legend style={at={(0.6,0.5)}, anchor=west, font=\footnotesize},
    width=\textwidth,
    height=4.5cm,
    cycle list name=color list,
    table/col sep=comma
]

% Shaded areas
\addplot[name path=test_q1, draw=none, forget plot] table[x=epoch, y=test_q1] {csv/plot_data_non_iid.csv};
\addplot[name path=test_q3, draw=none, forget plot] table[x=epoch, y=test_q3] {csv/plot_data_non_iid.csv};
\addplot[red!20, fill opacity=0.5] fill between[of=test_q1 and test_q3];

\addplot[name path=malicious_q1, draw=none, forget plot] table[x=epoch, y=malicious_q1] {csv/plot_data_non_iid.csv};
\addplot[name path=malicious_q3, draw=none, forget plot] table[x=epoch, y=malicious_q3] {csv/plot_data_non_iid.csv};
\addplot[blue!20, fill opacity=0.5] fill between[of=malicious_q1 and malicious_q3];

\addplot[name path=baseline_q1, draw=none, forget plot] table[x=epoch, y=baseline_q1] {csv/plot_data_non_iid.csv};
\addplot[name path=baseline_q3, draw=none, forget plot] table[x=epoch, y=baseline_q3] {csv/plot_data_non_iid.csv};
\addplot[green!20, fill opacity=0.5] fill between[of=baseline_q1 and baseline_q3];

% Medians
\addplot[red!50!white, thick] table[x=epoch, y=test_median] {csv/plot_data_non_iid.csv};
\addlegendentry{Test Data}

\addplot[blue!50!white, thick] table[x=epoch, y=malicious_median] {csv/plot_data_non_iid.csv};
\addlegendentry{Malicious Data}

\addplot[green!60!black!50!white, thick] table[x=epoch, y=baseline_median] {csv/plot_data_non_iid.csv};
\addlegendentry{$\mathcal{B}$}

\end{axis}
\end{tikzpicture}
\end{subfigure}
\caption{Ablation study using only natural forgetting.}
\label{fig:ablation_non_iid}
\end{subfigure}
\hfill
% Right (Non-IID)
\begin{subfigure}{0.49\textwidth}
\begin{subfigure}{0.49\textwidth}
\centering
\begin{tikzpicture}
\begin{axis}[
    title={IID},
    xlabel={Epochs},
    xmin=0, xmax=10,
    ymin=0, ymax=1,
    grid=major,
    grid style={dashed, gray!50},
    legend style={at={(0.55,0.55)}, anchor=west, font=\small},
    width=\textwidth,
    height=4.5cm,
    cycle list name=color list,
    table/col sep=comma
]

% Shaded areas
\addplot[name path=test_q1, draw=none, forget plot] table[x=epoch, y=test_q1] {csv/plot_data_iid_random.csv};
\addplot[name path=test_q3, draw=none, forget plot] table[x=epoch, y=test_q3] {csv/plot_data_iid_random.csv};
\addplot[red!20, fill opacity=0.5] fill between[of=test_q1 and test_q3];

\addplot[name path=malicious_q1, draw=none, forget plot] table[x=epoch, y=malicious_q1] {csv/plot_data_iid_random.csv};
\addplot[name path=malicious_q3, draw=none, forget plot] table[x=epoch, y=malicious_q3] {csv/plot_data_iid_random.csv};
\addplot[blue!20, fill opacity=0.5] fill between[of=malicious_q1 and malicious_q3];

\addplot[name path=baseline_q1, draw=none, forget plot] table[x=epoch, y=baseline_q1] {csv/plot_data_iid_random.csv};
\addplot[name path=baseline_q3, draw=none, forget plot] table[x=epoch, y=baseline_q3] {csv/plot_data_iid_random.csv};
\addplot[green!20, fill opacity=0.5] fill between[of=baseline_q1 and baseline_q3];

% Medians
\addplot[red!50!white, thick] table[x=epoch, y=test_median] {csv/plot_data_iid_random.csv};
%\addlegendentry{Test Data}

\addplot[blue!50!white, thick] table[x=epoch, y=malicious_median] {csv/plot_data_iid_random.csv};
%\addlegendentry{Malicious Data}

\addplot[green!60!black!50!white, thick] table[x=epoch, y=baseline_median] {csv/plot_data_iid_random.csv};
%\addlegendentry{$\mathcal{B}$}

\end{axis}
\end{tikzpicture}
\end{subfigure}
\hfill
% Right (Non-IID)
\begin{subfigure}{0.49\textwidth}
\centering
\begin{tikzpicture}
\begin{axis}[
    title={Non-IID},
    xlabel={Epochs},
    xmin=0, xmax=10,
    ymin=0, ymax=1,
    grid=major,
    grid style={dashed, gray!50},
    legend style={at={(0.45,0.5)}, anchor=west, font=\footnotesize},
    width=\textwidth,
    height=4.5cm,
    cycle list name=color list,
    table/col sep=comma
]

% Shaded areas
\addplot[name path=test_q1, draw=none, forget plot] table[x=epoch, y=test_q1] {csv/plot_data_non_iid_random.csv};
\addplot[name path=test_q3, draw=none, forget plot] table[x=epoch, y=test_q3] {csv/plot_data_non_iid_random.csv};
\addplot[red!20, fill opacity=0.5] fill between[of=test_q1 and test_q3];

\addplot[name path=malicious_q1, draw=none, forget plot] table[x=epoch, y=malicious_q1] {csv/plot_data_non_iid_random.csv};
\addplot[name path=malicious_q3, draw=none, forget plot] table[x=epoch, y=malicious_q3] {csv/plot_data_non_iid_random.csv};
\addplot[blue!20, fill opacity=0.5] fill between[of=malicious_q1 and malicious_q3];

\addplot[name path=baseline_q1, draw=none, forget plot] table[x=epoch, y=baseline_q1] {csv/plot_data_non_iid_random.csv};
\addplot[name path=baseline_q3, draw=none, forget plot] table[x=epoch, y=baseline_q3] {csv/plot_data_non_iid_random.csv};
\addplot[green!20, fill opacity=0.5] fill between[of=baseline_q1 and baseline_q3];

% Medians
\addplot[red!50!white, thick] table[x=epoch, y=test_median] {csv/plot_data_non_iid_random.csv};
%\addlegendentry{Test Data}

\addplot[blue!50!white, thick] table[x=epoch, y=malicious_median] {csv/plot_data_non_iid_random.csv};
%\addlegendentry{Malicious Data}

\addplot[green!60!black!50!white, thick] table[x=epoch, y=baseline_median] {csv/plot_data_non_iid_random.csv};
%\addlegendentry{$\mathcal{B}$}

\end{axis}
\end{tikzpicture}
\end{subfigure}
\caption{Ablation study using random pruning.}
\label{fig:ablation_random}
\end{subfigure}\\
\begin{subfigure}{0.49\textwidth}
\begin{subfigure}{0.49\textwidth}
\centering
\begin{tikzpicture}
\begin{axis}[
    title={IID},
    xlabel={Epochs},
    xmin=0, xmax=10,
    ymin=0, ymax=1,
    grid=major,
    grid style={dashed, gray!50},
    legend style={at={(0.55,0.55)}, anchor=west, font=\small},
    width=\textwidth,
    height=4.5cm,
    cycle list name=color list,
    table/col sep=comma
]

\addplot[black, dashed, thick, domain=0:10] {0.74};
% Shaded areas
\addplot[name path=test_q1, draw=none, forget plot] table[x=epoch, y=test_q1] {csv/plot_data_iid_highest.csv};
\addplot[name path=test_q3, draw=none, forget plot] table[x=epoch, y=test_q3] {csv/plot_data_iid_highest.csv};
\addplot[red!20, fill opacity=0.5] fill between[of=test_q1 and test_q3];

\addplot[name path=malicious_q1, draw=none, forget plot] table[x=epoch, y=malicious_q1] {csv/plot_data_iid_highest.csv};
\addplot[name path=malicious_q3, draw=none, forget plot] table[x=epoch, y=malicious_q3] {csv/plot_data_iid_highest.csv};
\addplot[blue!20, fill opacity=0.5] fill between[of=malicious_q1 and malicious_q3];

\addplot[name path=baseline_q1, draw=none, forget plot] table[x=epoch, y=baseline_q1] {csv/plot_data_iid_highest.csv};
\addplot[name path=baseline_q3, draw=none, forget plot] table[x=epoch, y=baseline_q3] {csv/plot_data_iid_highest.csv};
\addplot[green!20, fill opacity=0.5] fill between[of=baseline_q1 and baseline_q3];

% Medians
\addplot[red!50!white, thick] table[x=epoch, y=test_median] {csv/plot_data_iid_highest.csv};
%\addlegendentry{Test Data}

\addplot[blue!50!white, thick] table[x=epoch, y=malicious_median] {csv/plot_data_iid_highest.csv};
%\addlegendentry{Malicious Data}

\addplot[green!60!black!50!white, thick] table[x=epoch, y=baseline_median] {csv/plot_data_iid_highest.csv};
%\addlegendentry{$\mathcal{B}$}

\end{axis}
\end{tikzpicture}
\end{subfigure}
\hfill
% Right (Non-IID)
\begin{subfigure}{0.49\textwidth}
\centering
\begin{tikzpicture}
\begin{axis}[
    title={Non-IID},
    xlabel={Epochs},
    xmin=0, xmax=10,
    ymin=0, ymax=1,
    grid=major,
    grid style={dashed, gray!50},
    legend style={at={(0.45,0.5)}, anchor=west, font=\footnotesize},
    width=\textwidth,
    height=4.5cm,
    cycle list name=color list,
    table/col sep=comma
]

\addplot[black, dashed, thick, domain=0:10] {0.74};

% Shaded areas
\addplot[name path=test_q1, draw=none, forget plot] table[x=epoch, y=test_q1] {csv/plot_data_non_iid_highest.csv};
\addplot[name path=test_q3, draw=none, forget plot] table[x=epoch, y=test_q3] {csv/plot_data_non_iid_highest.csv};
\addplot[red!20, fill opacity=0.5] fill between[of=test_q1 and test_q3];

\addplot[name path=malicious_q1, draw=none, forget plot] table[x=epoch, y=malicious_q1] {csv/plot_data_non_iid_highest.csv};
\addplot[name path=malicious_q3, draw=none, forget plot] table[x=epoch, y=malicious_q3] {csv/plot_data_non_iid_highest.csv};
\addplot[blue!20, fill opacity=0.5] fill between[of=malicious_q1 and malicious_q3];

\addplot[name path=baseline_q1, draw=none, forget plot] table[x=epoch, y=baseline_q1] {csv/plot_data_non_iid_highest.csv};
\addplot[name path=baseline_q3, draw=none, forget plot] table[x=epoch, y=baseline_q3] {csv/plot_data_non_iid_highest.csv};
\addplot[green!20, fill opacity=0.5] fill between[of=baseline_q1 and baseline_q3];

% Medians
\addplot[red!50!white, thick] table[x=epoch, y=test_median] {csv/plot_data_non_iid_highest.csv};
%\addlegendentry{Test Data}

\addplot[blue!50!white, thick] table[x=epoch, y=malicious_median] {csv/plot_data_non_iid_highest.csv};
%\addlegendentry{Malicious Data}

\addplot[green!60!black!50!white, thick] table[x=epoch, y=baseline_median] {csv/plot_data_non_iid_highest.csv};
%\addlegendentry{$\mathcal{B}$}

\end{axis}
\end{tikzpicture}
\end{subfigure}
\caption{Ablation study using targeted pruning based on the highest malicious weight magnitudes.}
\label{fig:ablation_targeted}
\end{subfigure}
\caption{(a) Effect of natural forgetting after removing malicious clients without using FedUP. (b) Effect of random pruning for unlearning. (c) Effect of targeted pruning based on the highest malicious magnitude weights. The plots show accuracy on test data (red), accuracy on malicious data (blue), and baseline malicious accuracy after full retraining (green). Results are averaged across varying percentages of malicious clients.}
\label{fig:ablation_combined}
\end{figure*}

\subsubsection{Comparison with SOTA}
Table~\ref{tab:comparison_backdoor} shows the comparison with NoT under the backdoor attack setting. The accuracy column on malicious data is omitted because both algorithms effectively remove malicious contributions (for FedUP results, refer to Table~\ref{tab:backdoor}). The main difference lies in how efficiently each solution recovers the original model performance. FedUP restores the performance lost during pruning within a few rounds, while NoT struggles to do so and, in some cases, fails to fully recover even after many rounds. This trend is especially evident in the Non-IID setting. This is because FedUP targets specific malicious weights, whereas NoT disrupts the entire first layer indiscriminately. Such broad disruption is highly detrimental, requiring significantly more recovery rounds to regain original performance.

Table~\ref{tab:comparison_benign} presents, instead, the comparison with NoT and FedEraser in the benign setting. The pruning percentage selected for FedUP is $\mathcal{P} = 30\%$. This is a particularly unique configuration, as it is both IID and does not involve malicious clients or clients with distinct data. Even in this context, the percentage of weights to prune is consistent with expectations. 
All algorithms successfully unlearn the client's data, as evidenced by the significant drop in the global model's performance on that client's data. Additionally, they can recover the performance lost during the unlearning process, restoring the model’s effectiveness on retained data. However, the key difference is that while FedEraser and NoT fail to rapidly recover the original performance, FedUP fully restores the model's performance in just a few recovery rounds.
In contrast, bot NoT and FedEraser require at least 20 additional training rounds to achieve similar results. NoT even fails to fully recover the lost performance, as the test accuracy after the recovery rounds remains below the original level. Moreover, FedUP requires only $1/20$ of the storage needed by FedEraser to execute, as it only requires the updates and the global model from the last training round. In this regard, NoT is the most efficient solution since it requires only the global model, although its efficiency is not significantly different from that of FedUP. 
This comparison highlights how our algorithm is storage-efficient and significantly faster than SOTA solutions while remaining equally effective. Furthermore, this demonstrates that, despite being specifically designed for malicious environments, FedUP can also be effectively applied in benign settings as well. }}

\subsubsection{Ablation Study}
\label{sec:ablation}
{\color{black}To assess the impact of FedUP, we conducted three ablation studies, shown in Figure~\ref{fig:ablation_combined}. In the first study, malicious clients are excluded from training, but no pruning is applied, so the only mitigating factor is natural forgetting. In the second study, a percentage of weights is randomly pruned, following the same threshold guidelines used in our experiments, but without selecting weights based on their contribution. The third study applies the same pruning percentage but targets the weights with the highest estimated malicious magnitude, simulating a basic yet focused unlearning strategy. All experiments use the previously described backdoor attack setting and evaluate various proportions of malicious clients under both IID and Non-IID data distributions. Due to space constraints, results are aggregated, with variance across the different configurations shown separately for the IID and Non-IID cases. In the figures, red indicates accuracy on test data, blue shows accuracy on malicious data, and green represents the baseline $\mathcal{B}$ accuracy on malicious data achieved through full retraining. At epoch 0, the malicious clients are removed, and 10 subsequent epochs represent the recovery rounds.

As shown, natural forgetting can initially reduce some of the malicious influence on the model weights. However, consistent with findings in prior work \cite{romandini2024federated}, it is not sufficient to completely eliminate the impact of malicious updates. Random pruning can help improve unlearning performance, but its effectiveness remains limited. This is likely because the pruning percentage is limited and the weights removed are often either benign or insignificant. Nevertheless, these two approaches fail to fully eliminate malicious contributions. Targeted pruning based on the highest magnitude weights produces better results, confirming that selecting meaningful weights is crucial. Although this is the only method that effectively removes malicious contributions, it struggles to quickly restore performance, typically requiring between 8 and 10 rounds, which is significantly more than FedUP and sometimes failing to reach the original accuracy. This is because it focuses solely on the malicious weights without considering benign ones, which risks removing weights that are important for benign clients. These experiments demonstrate that each component of FedUP is essential for effective and reliable unlearning of malicious contributions.
}

\section{Conclusion}
\label{sec:conclusion}
%In this paper, we introduce FedUP, a novel FU algorithm that can efficiently and effectively remove malicious contributions through pruning. FedUP works by zeroing out the highest magnitude weights that differ the most between benign and malicious clients. To recover the performance lost during pruning, FedUP performs a few additional training rounds. Our extensive experiments, conducted on both IID and Non-IID settings with multiple datasets and models, demonstrate that FedUP effectively reduces the influence of malicious contributions while efficiently restoring the model’s performance on benign data. Furthermore, FedUP proves to be faster than SOTA FU solutions, achieving full recovery in significantly fewer rounds and using considerably less storage. These results showcase FedUP as an efficient approach to securing FL models in environments where adversaries attempt to poison the global model.
In this paper, we introduced FedUP, a novel FU algorithm specifically designed to effectively and efficiently remove malicious contributions in adversarial FL environments. Unlike conventional FU methods that assume cooperative clients, FedUP operates under a strong adversarial threat model where up to 50\% - 1 of clients may be malicious and colluding. Our approach leverages a lightweight pruning mechanism that zeros out the highest-magnitude weights most dissimilar between benign and malicious client updates, effectively isolating malicious influence despite overlapping updates. To recover from the performance loss caused by pruning, FedUP performs a limited number of training rounds, avoiding costly retraining from scratch. Extensive experiments across IID and Non-IID settings, multiple models, and two attack types (label-flipping and backdoor) show that FedUP reliably reduces malicious impact and restores performance on benign data. Additionally, it outperforms SOTA FU techniques in both speed and storage efficiency, making it a practical solution for real-world FL deployments.
\bibliographystyle{ieeetr}
\bibliography{main}
\end{document}

%% file: histogram.tex
% This file was created with tikzplotlib v0.10.1.
\begin{tikzpicture}

\definecolor{darkgray176}{RGB}{176,176,176}
\definecolor{darkorange25512714}{RGB}{255,127,14}
\definecolor{steelblue31119180}{RGB}{31,119,180}

\begin{axis}[
tick align=outside,
tick pos=left,
x grid style={darkgray176},
xlabel={Difference},
xmin=0, xmax=0.4,
xtick style={color=black},
y grid style={darkgray176},
ylabel={Count},
ymin=0, ymax=100,
ytick style={color=black},
height=0.6\linewidth,
width=\linewidth,
legend cell align={left},
clip=true
]
\draw[draw=none,fill=steelblue31119180,fill opacity=0.4] (axis cs:0,0) rectangle (axis cs:0.00999999977648258,100);
\addlegendimage{ybar,ybar legend,draw=none,fill=steelblue31119180,fill opacity=0.4}
\addlegendentry{Non-IID}% - Skew: 96}

\draw[draw=none,fill=steelblue31119180,fill opacity=0.4] (axis cs:0.00999999977648258,0) rectangle (axis cs:0.0199999995529652,113);
\draw[draw=none,fill=steelblue31119180,fill opacity=0.4] (axis cs:0.0199999995529652,0) rectangle (axis cs:0.0299999993294477,69);
\draw[draw=none,fill=steelblue31119180,fill opacity=0.4] (axis cs:0.0299999993294477,0) rectangle (axis cs:0.0399999991059303,48);
\draw[draw=none,fill=steelblue31119180,fill opacity=0.4] (axis cs:0.0399999991059303,0) rectangle (axis cs:0.0500000007450581,27);
\draw[draw=none,fill=steelblue31119180,fill opacity=0.4] (axis cs:0.0500000007450581,0) rectangle (axis cs:0.0599999986588955,26);
\draw[draw=none,fill=steelblue31119180,fill opacity=0.4] (axis cs:0.0599999986588955,0) rectangle (axis cs:0.0700000002980232,21);
\draw[draw=none,fill=steelblue31119180,fill opacity=0.4] (axis cs:0.0700000002980232,0) rectangle (axis cs:0.0799999982118607,18);
\draw[draw=none,fill=steelblue31119180,fill opacity=0.4] (axis cs:0.0799999982118607,0) rectangle (axis cs:0.0900000035762787,29);
\draw[draw=none,fill=steelblue31119180,fill opacity=0.4] (axis cs:0.0900000035762787,0) rectangle (axis cs:0.100000001490116,19);
\draw[draw=none,fill=steelblue31119180,fill opacity=0.4] (axis cs:0.100000001490116,0) rectangle (axis cs:0.109999999403954,12);
\draw[draw=none,fill=steelblue31119180,fill opacity=0.4] (axis cs:0.109999999403954,0) rectangle (axis cs:0.119999997317791,9);
\draw[draw=none,fill=steelblue31119180,fill opacity=0.4] (axis cs:0.119999997317791,0) rectangle (axis cs:0.129999995231628,12);
\draw[draw=none,fill=steelblue31119180,fill opacity=0.4] (axis cs:0.129999995231628,0) rectangle (axis cs:0.140000000596046,8);
\draw[draw=none,fill=steelblue31119180,fill opacity=0.4] (axis cs:0.140000000596046,0) rectangle (axis cs:0.150000005960464,13);
\draw[draw=none,fill=steelblue31119180,fill opacity=0.4] (axis cs:0.150000005960464,0) rectangle (axis cs:0.159999996423721,7);
\draw[draw=none,fill=steelblue31119180,fill opacity=0.4] (axis cs:0.159999996423721,0) rectangle (axis cs:0.170000001788139,8);
\draw[draw=none,fill=steelblue31119180,fill opacity=0.4] (axis cs:0.170000001788139,0) rectangle (axis cs:0.180000007152557,4);
\draw[draw=none,fill=steelblue31119180,fill opacity=0.4] (axis cs:0.180000007152557,0) rectangle (axis cs:0.189999997615814,7);
\draw[draw=none,fill=steelblue31119180,fill opacity=0.4] (axis cs:0.189999997615814,0) rectangle (axis cs:0.200000002980232,4);
\draw[draw=none,fill=steelblue31119180,fill opacity=0.4] (axis cs:0.200000002980232,0) rectangle (axis cs:0.209999993443489,4);
\draw[draw=none,fill=steelblue31119180,fill opacity=0.4] (axis cs:0.209999993443489,0) rectangle (axis cs:0.219999998807907,3);
\draw[draw=none,fill=steelblue31119180,fill opacity=0.4] (axis cs:0.219999998807907,0) rectangle (axis cs:0.230000004172325,2);
\draw[draw=none,fill=steelblue31119180,fill opacity=0.4] (axis cs:0.230000004172325,0) rectangle (axis cs:0.239999994635582,5);
\draw[draw=none,fill=steelblue31119180,fill opacity=0.4] (axis cs:0.239999994635582,0) rectangle (axis cs:0.25,3);
\draw[draw=none,fill=steelblue31119180,fill opacity=0.4] (axis cs:0.25,0) rectangle (axis cs:0.259999990463257,0);
\draw[draw=none,fill=steelblue31119180,fill opacity=0.4] (axis cs:0.259999990463257,0) rectangle (axis cs:0.270000010728836,6);
\draw[draw=none,fill=steelblue31119180,fill opacity=0.4] (axis cs:0.270000010728836,0) rectangle (axis cs:0.280000001192093,3);
\draw[draw=none,fill=steelblue31119180,fill opacity=0.4] (axis cs:0.280000001192093,0) rectangle (axis cs:0.28999999165535,2);
\draw[draw=none,fill=steelblue31119180,fill opacity=0.4] (axis cs:0.28999999165535,0) rectangle (axis cs:0.300000011920929,1);
\draw[draw=none,fill=steelblue31119180,fill opacity=0.4] (axis cs:0.300000011920929,0) rectangle (axis cs:0.310000002384186,3);
\draw[draw=none,fill=steelblue31119180,fill opacity=0.4] (axis cs:0.310000002384186,0) rectangle (axis cs:0.319999992847443,0);
\draw[draw=none,fill=steelblue31119180,fill opacity=0.4] (axis cs:0.319999992847443,0) rectangle (axis cs:0.330000013113022,0);
\draw[draw=none,fill=steelblue31119180,fill opacity=0.4] (axis cs:0.330000013113022,0) rectangle (axis cs:0.340000003576279,2);
\draw[draw=none,fill=steelblue31119180,fill opacity=0.4] (axis cs:0.340000003576279,0) rectangle (axis cs:0.349999994039536,0);
\draw[draw=none,fill=steelblue31119180,fill opacity=0.4] (axis cs:0.349999994039536,0) rectangle (axis cs:0.360000014305115,2);
\draw[draw=none,fill=steelblue31119180,fill opacity=0.4] (axis cs:0.360000014305115,0) rectangle (axis cs:0.370000004768372,1);
\draw[draw=none,fill=steelblue31119180,fill opacity=0.4] (axis cs:0.370000004768372,0) rectangle (axis cs:0.379999995231628,5);
\draw[draw=none,fill=steelblue31119180,fill opacity=0.4] (axis cs:0.379999995231628,0) rectangle (axis cs:0.389999985694885,1);
\draw[draw=none,fill=steelblue31119180,fill opacity=0.4] (axis cs:0.389999985694885,0) rectangle (axis cs:0.400000005960464,0);
\draw[draw=none,fill=steelblue31119180,fill opacity=0.4] (axis cs:0.400000005960464,0) rectangle (axis cs:0.409999996423721,0);
\draw[draw=none,fill=steelblue31119180,fill opacity=0.4] (axis cs:0.409999996423721,0) rectangle (axis cs:0.419999986886978,2);
\draw[draw=none,fill=steelblue31119180,fill opacity=0.4] (axis cs:0.419999986886978,0) rectangle (axis cs:0.430000007152557,0);
\draw[draw=none,fill=steelblue31119180,fill opacity=0.4] (axis cs:0.430000007152557,0) rectangle (axis cs:0.439999997615814,0);
\draw[draw=none,fill=steelblue31119180,fill opacity=0.4] (axis cs:0.439999997615814,0) rectangle (axis cs:0.449999988079071,0);
\draw[draw=none,fill=steelblue31119180,fill opacity=0.4] (axis cs:0.449999988079071,0) rectangle (axis cs:0.46000000834465,0);
\draw[draw=none,fill=steelblue31119180,fill opacity=0.4] (axis cs:0.46000000834465,0) rectangle (axis cs:0.469999998807907,1);
\draw[draw=none,fill=steelblue31119180,fill opacity=0.4] (axis cs:0.469999998807907,0) rectangle (axis cs:0.479999989271164,1);
\draw[draw=none,fill=steelblue31119180,fill opacity=0.4] (axis cs:0.479999989271164,0) rectangle (axis cs:0.490000009536743,0);
\draw[draw=none,fill=steelblue31119180,fill opacity=0.4] (axis cs:0.490000009536743,0) rectangle (axis cs:0.5,1);
\draw[draw=none,fill=steelblue31119180,fill opacity=0.4] (axis cs:0.5,0) rectangle (axis cs:0.509999990463257,2);
\draw[draw=none,fill=darkorange25512714,fill opacity=0.4] (axis cs:0,0) rectangle (axis cs:0.00999999977648258,100);

\addlegendimage{ybar,ybar legend,draw=none,fill=darkorange25512714,fill opacity=0.4}
\addlegendentry{IID}% - Skew: 158}

\draw[draw=none,fill=darkorange25512714,fill opacity=0.4] (axis cs:0.00999999977648258,0) rectangle (axis cs:0.0199999995529652,143);
\draw[draw=none,fill=darkorange25512714,fill opacity=0.4] (axis cs:0.0199999995529652,0) rectangle (axis cs:0.0299999993294477,65);
\draw[draw=none,fill=darkorange25512714,fill opacity=0.4] (axis cs:0.0299999993294477,0) rectangle (axis cs:0.0399999991059303,51);
\draw[draw=none,fill=darkorange25512714,fill opacity=0.4] (axis cs:0.0399999991059303,0) rectangle (axis cs:0.0500000007450581,34);
\draw[draw=none,fill=darkorange25512714,fill opacity=0.4] (axis cs:0.0500000007450581,0) rectangle (axis cs:0.0599999986588955,19);
\draw[draw=none,fill=darkorange25512714,fill opacity=0.4] (axis cs:0.0599999986588955,0) rectangle (axis cs:0.0700000002980232,24);
\draw[draw=none,fill=darkorange25512714,fill opacity=0.4] (axis cs:0.0700000002980232,0) rectangle (axis cs:0.0799999982118607,17);
\draw[draw=none,fill=darkorange25512714,fill opacity=0.4] (axis cs:0.0799999982118607,0) rectangle (axis cs:0.0900000035762787,14);
\draw[draw=none,fill=darkorange25512714,fill opacity=0.4] (axis cs:0.0900000035762787,0) rectangle (axis cs:0.100000001490116,8);
\draw[draw=none,fill=darkorange25512714,fill opacity=0.4] (axis cs:0.100000001490116,0) rectangle (axis cs:0.109999999403954,7);
\draw[draw=none,fill=darkorange25512714,fill opacity=0.4] (axis cs:0.109999999403954,0) rectangle (axis cs:0.119999997317791,7);
\draw[draw=none,fill=darkorange25512714,fill opacity=0.4] (axis cs:0.119999997317791,0) rectangle (axis cs:0.129999995231628,5);
\draw[draw=none,fill=darkorange25512714,fill opacity=0.4] (axis cs:0.129999995231628,0) rectangle (axis cs:0.140000000596046,5);
\draw[draw=none,fill=darkorange25512714,fill opacity=0.4] (axis cs:0.140000000596046,0) rectangle (axis cs:0.150000005960464,4);
\draw[draw=none,fill=darkorange25512714,fill opacity=0.4] (axis cs:0.150000005960464,0) rectangle (axis cs:0.159999996423721,1);
\draw[draw=none,fill=darkorange25512714,fill opacity=0.4] (axis cs:0.159999996423721,0) rectangle (axis cs:0.170000001788139,2);
\draw[draw=none,fill=darkorange25512714,fill opacity=0.4] (axis cs:0.170000001788139,0) rectangle (axis cs:0.180000007152557,4);
\draw[draw=none,fill=darkorange25512714,fill opacity=0.4] (axis cs:0.180000007152557,0) rectangle (axis cs:0.189999997615814,2);
\draw[draw=none,fill=darkorange25512714,fill opacity=0.4] (axis cs:0.189999997615814,0) rectangle (axis cs:0.200000002980232,1);
\draw[draw=none,fill=darkorange25512714,fill opacity=0.4] (axis cs:0.200000002980232,0) rectangle (axis cs:0.209999993443489,4);
\draw[draw=none,fill=darkorange25512714,fill opacity=0.4] (axis cs:0.209999993443489,0) rectangle (axis cs:0.219999998807907,1);
\draw[draw=none,fill=darkorange25512714,fill opacity=0.4] (axis cs:0.219999998807907,0) rectangle (axis cs:0.230000004172325,2);
\draw[draw=none,fill=darkorange25512714,fill opacity=0.4] (axis cs:0.230000004172325,0) rectangle (axis cs:0.239999994635582,0);
\draw[draw=none,fill=darkorange25512714,fill opacity=0.4] (axis cs:0.239999994635582,0) rectangle (axis cs:0.25,1);
\draw[draw=none,fill=darkorange25512714,fill opacity=0.4] (axis cs:0.25,0) rectangle (axis cs:0.259999990463257,1);
\draw[draw=none,fill=darkorange25512714,fill opacity=0.4] (axis cs:0.259999990463257,0) rectangle (axis cs:0.270000010728836,0);
\draw[draw=none,fill=darkorange25512714,fill opacity=0.4] (axis cs:0.270000010728836,0) rectangle (axis cs:0.280000001192093,1);
\draw[draw=none,fill=darkorange25512714,fill opacity=0.4] (axis cs:0.280000001192093,0) rectangle (axis cs:0.28999999165535,0);
\draw[draw=none,fill=darkorange25512714,fill opacity=0.4] (axis cs:0.28999999165535,0) rectangle (axis cs:0.300000011920929,0);
\draw[draw=none,fill=darkorange25512714,fill opacity=0.4] (axis cs:0.300000011920929,0) rectangle (axis cs:0.310000002384186,2);
\draw[draw=none,fill=darkorange25512714,fill opacity=0.4] (axis cs:0.310000002384186,0) rectangle (axis cs:0.319999992847443,0);
\draw[draw=none,fill=darkorange25512714,fill opacity=0.4] (axis cs:0.319999992847443,0) rectangle (axis cs:0.330000013113022,1);
\draw[draw=none,fill=darkorange25512714,fill opacity=0.4] (axis cs:0.330000013113022,0) rectangle (axis cs:0.340000003576279,0);
\draw[draw=none,fill=darkorange25512714,fill opacity=0.4] (axis cs:0.340000003576279,0) rectangle (axis cs:0.349999994039536,0);
\draw[draw=none,fill=darkorange25512714,fill opacity=0.4] (axis cs:0.349999994039536,0) rectangle (axis cs:0.360000014305115,0);
\draw[draw=none,fill=darkorange25512714,fill opacity=0.4] (axis cs:0.360000014305115,0) rectangle (axis cs:0.370000004768372,0);
\draw[draw=none,fill=darkorange25512714,fill opacity=0.4] (axis cs:0.370000004768372,0) rectangle (axis cs:0.379999995231628,0);
\draw[draw=none,fill=darkorange25512714,fill opacity=0.4] (axis cs:0.379999995231628,0) rectangle (axis cs:0.389999985694885,0);
\draw[draw=none,fill=darkorange25512714,fill opacity=0.4] (axis cs:0.389999985694885,0) rectangle (axis cs:0.400000005960464,0);
\draw[draw=none,fill=darkorange25512714,fill opacity=0.4] (axis cs:0.400000005960464,0) rectangle (axis cs:0.409999996423721,0);
\draw[draw=none,fill=darkorange25512714,fill opacity=0.4] (axis cs:0.409999996423721,0) rectangle (axis cs:0.419999986886978,0);
\draw[draw=none,fill=darkorange25512714,fill opacity=0.4] (axis cs:0.419999986886978,0) rectangle (axis cs:0.430000007152557,0);
\draw[draw=none,fill=darkorange25512714,fill opacity=0.4] (axis cs:0.430000007152557,0) rectangle (axis cs:0.439999997615814,0);
\draw[draw=none,fill=darkorange25512714,fill opacity=0.4] (axis cs:0.439999997615814,0) rectangle (axis cs:0.449999988079071,1);
\draw[draw=none,fill=darkorange25512714,fill opacity=0.4] (axis cs:0.449999988079071,0) rectangle (axis cs:0.46000000834465,1);
\draw[draw=none,fill=darkorange25512714,fill opacity=0.4] (axis cs:0.46000000834465,0) rectangle (axis cs:0.469999998807907,0);
\draw[draw=none,fill=darkorange25512714,fill opacity=0.4] (axis cs:0.469999998807907,0) rectangle (axis cs:0.479999989271164,0);
\draw[draw=none,fill=darkorange25512714,fill opacity=0.4] (axis cs:0.479999989271164,0) rectangle (axis cs:0.490000009536743,0);
\draw[draw=none,fill=darkorange25512714,fill opacity=0.4] (axis cs:0.490000009536743,0) rectangle (axis cs:0.5,0);
\draw[draw=none,fill=darkorange25512714,fill opacity=0.4] (axis cs:0.5,0) rectangle (axis cs:0.509999990463257,0);
\end{axis}

\end{tikzpicture}